\documentclass[10pt]{article}

\usepackage{arxiv_twocol}

\usepackage{latexsym}

\usepackage[T1]{fontenc}

\usepackage[utf8]{inputenc}

\usepackage{microtype}

\usepackage{graphicx}
\usepackage{subcaption}
\usepackage{booktabs}
\usepackage{multirow}
\usepackage{makecell}
\usepackage{tabularx}
\usepackage{amsmath}
\usepackage{bm}
\usepackage{amsthm}

\newtheorem{proposition}{Proposition}
\newcommand{\Rmnum}[1]{\uppercase\expandafter{\romannumeral #1}}

\title{CoDe-LoRA: Mitigating the Orthogonality Dilemma in Continual Learning of LLMs via Knowledge Consolidation and Decoupling}

\author{
  Maoqi Liu$^{1}$,
  Quan Fang$^{1,\text{\Letter}}$,
  Yufei He$^{2}$ \\
  $^{1}$Beijing University of Posts and Telecommunications \\
  $^{2}$National University of Singapore
}

\begin{document}
\maketitle
\begingroup\def\thefootnote{\Letter}\footnotetext{Corresponding author: \texttt{qfang@bupt.edu.cn}.}\endgroup
\begin{abstract}
Continual learning (CL) is essential for Large Language Models (LLMs) to sequentially adapt to evolving tasks. To mitigate catastrophic forgetting, recent advances implement low-rank adaptation with orthogonal projections (e.g., O-LoRA) to isolate task parameters. 
However, we reveal that such strict geometric constraints trigger an ``Orthogonality Dilemma'': rigid parameter isolation impedes the transfer and accumulation of shared representations across semantically related tasks. 
In this work, we propose a new replay-free method, called \underline{Co}nsolidation and \underline{De}coupling \underline{LoRA} (CoDe-LoRA), for CL of LLMs. CoDe-LoRA disentangles the learning process into \textit{Consolidating Universal Knowledge} and \textit{Decoupling Task-Specific Knowledge}. To achieve this, CoDe-LoRA leverages an adaptive null space projection mechanism and semantic routing to balance knowledge accumulation with task-specific adaptation. Experimental results across four backbones and three CL benchmarks show that CoDe-LoRA achieves the best average accuracy. Our code is available at \url{https://github.com/Estrellajer/CoDe-LoRA}.
\end{abstract}


\begin{figure*}[!ht]
  \centering
  \includegraphics[width=0.98\textwidth]{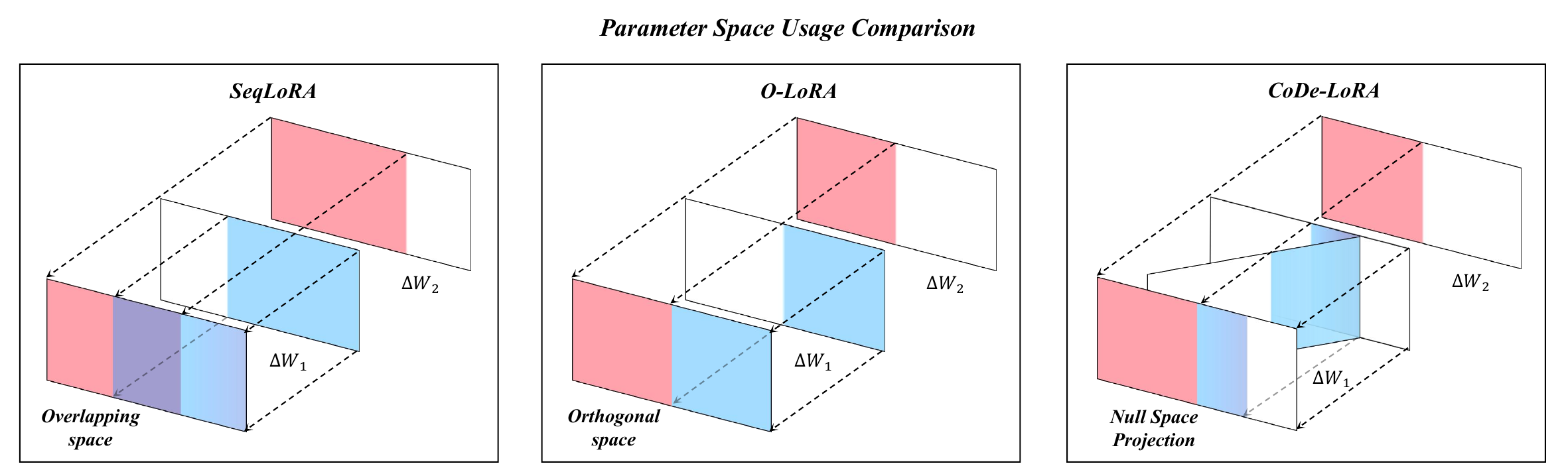}
  \caption{Conceptual illustration of parameter space utilization. Traditional LoRA suffers from interference, while O-LoRA's strict orthogonality hinders knowledge sharing. CoDe-LoRA leverages null space projection to consolidate shared knowledge.}
  \label{teaser}
\end{figure*}

\section{Introduction}

Large Language Models (LLMs) have redefined the boundaries of natural language processing, exhibiting remarkable capabilities across a wide range of tasks~\cite{wei2021finetuned}. However, adapting these models to a continuous stream of evolving tasks, known as \textit{Continual Learning (CL)}~\cite{de2021continual}, remains challenging. While parameter-efficient fine-tuning (PEFT) methods, particularly low-rank adaptation (LoRA)~\cite{hu2021lora}, have significantly lowered the cost of adaptation, naively applying them to sequential tasks leads to \textit{catastrophic forgetting}~\cite{li2017learning}, where learning new tasks degrades previously acquired capabilities on earlier tasks.

To mitigate forgetting, recent work combines PEFT with orthogonal projection strategies~\cite{saha2021gradient, wang2023orthogonal}. A representative example is O-LoRA~\cite{wang2023orthogonal}, which isolates task-specific parameters in mutually orthogonal subspaces. Although effective at reducing inter-task interference, such rigid isolation introduces a trade-off between interference mitigation and knowledge transfer. When sequential tasks share semantic structure, their optimal parameter updates should naturally overlap. Enforcing strict orthogonality removes these shared components, reducing representation reuse and weakening cross-task transfer. As illustrated in Figure~\ref{teaser}, this phenomenon reveals an ``Orthogonality Dilemma'': the same geometric constraint that protects previous knowledge can also restrict transferable representation learning across related tasks.

This limitation highlights a key challenge in continual adaptation: how to preserve prior knowledge while still accumulating transferable representations across related tasks. Motivated by this tension, we investigate two questions:
\begin{itemize}
    \item[\textbf{Q1.}] \textit{Why does strict orthogonality limit generalization across related tasks, and under what conditions is this limitation most severe?}
    \item[\textbf{Q2.}] \textit{How can we design a parameter-efficient framework that preserves transferable representations while maintaining task-specific expertise in a replay-free setting?}
\end{itemize}

In this paper, we introduce \underline{Co}nsolidation and \underline{De}coupling \underline{LoRA} (\textbf{CoDe-LoRA}), a dual-branch continual adaptation framework motivated by the orthogonality-induced transfer limitation. We provide a geometric characterization showing that strict orthogonal isolation becomes increasingly detrimental as task similarity grows.

CoDe-LoRA separates transferable representation consolidation from task-specific specialization through two complementary branches. Unlike strict orthogonal isolation methods that suppress overlapping representations entirely, our framework preserves transferable components while constraining interference-prone directions. Specifically, the \textit{Consolidation Branch (Co-LoRA)} maintains shared transferable representations through adaptive null-space projection with dynamic scaling to stabilize long-horizon continual adaptation, as empirically verified in Appendix~\ref{sec:appendix:scaling} (Table~\ref{tab:scaling_empirical}). Meanwhile, the \textit{Decoupling Branch (De-LoRA)} maintains a pool of independent LoRA experts together with a prototype-based semantic routing mechanism for robust inference without explicit task identifiers. Importantly, CoDe-LoRA operates without relying on historical data replay, avoiding additional storage overhead and privacy concerns.

Our main contributions are summarized as follows:
\begin{itemize}
\item We identify and analyze the \textbf{Orthogonality Dilemma} in PEFT-based continual learning, showing that strict parameter isolation becomes increasingly detrimental as task similarity grows.
\item We propose \textbf{CoDe-LoRA}, a dual-branch continual adaptation framework that separates transferable representation consolidation from task-specific specialization through adaptive null-space projection and prototype-based expert decoupling.
\item Extensive experiments across four backbones and three continual learning benchmarks show that CoDe-LoRA achieves the best average accuracy while balancing knowledge transfer and interference mitigation.
\end{itemize}

\begin{figure*}[!ht]
  \centering
  \includegraphics[width=\textwidth]{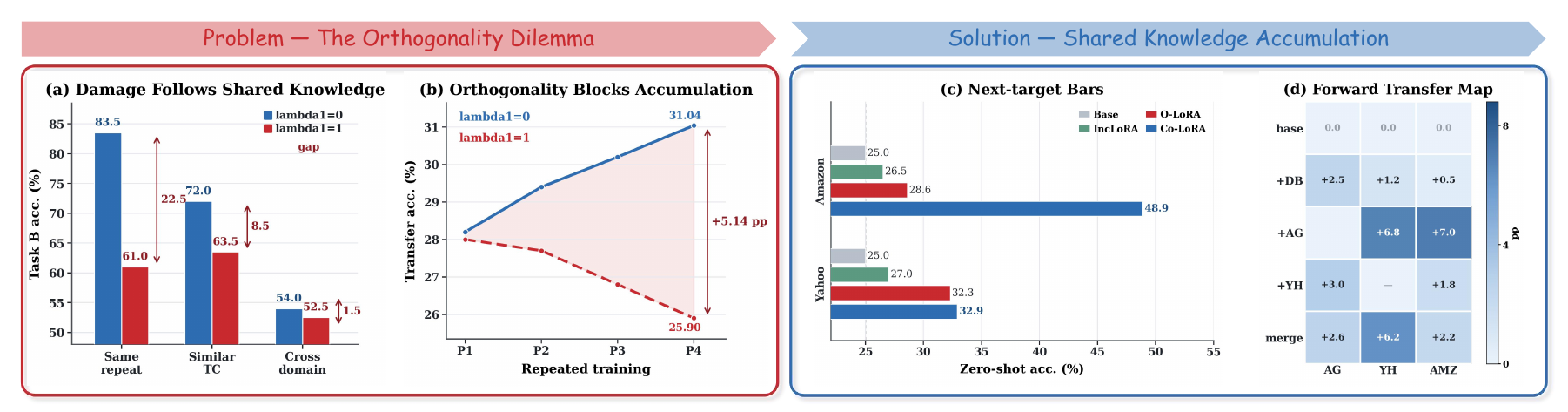}
    \caption{
    Preliminary evidence for the orthogonality dilemma and transferable representation accumulation.
    \textbf{(a)} Orthogonality-induced degradation under different task similarity levels.
    \textbf{(b)} Transfer trajectories over repeated training with and without orthogonal constraints.
    \textbf{(c)} Next-task zero-shot transfer comparison across continual adaptation methods.
    \textbf{(d)} Forward transfer dynamics across sequential tasks.
    }
  \label{fig:empirical}
\end{figure*}

\section{Related Work}
\label{sec:related_work}

\subsection{Parameter-Efficient Fine-Tuning}

Parameter-Efficient Fine-Tuning (PEFT) adapts Large Language Models (LLMs) with minimal additional parameters. Early approaches such as Adapters~\cite{houlsby2019parameter} and Prompt/Prefix Tuning~\cite{li2021prefix, lester2021power} introduced lightweight modules while freezing the backbone. LoRA~\cite{hu2021lora} further improves efficiency through low-rank matrices. Although effective for static tasks, sequential adaptation with PEFT methods often leads to catastrophic forgetting of earlier tasks as training proceeds.

\subsection{Continual Fine-Tuning for LLMs}

Continual fine-tuning methods mainly rely on regularization~\cite{kirkpatrick2017overcoming}, replay~\cite{qin2021lfpt5, he2024seekr}, or parameter expansion~\cite{razdaibiedina2023progressive}. Expansion-based methods avoid interference through task-specific parameters, but their parameter cost grows with the number of tasks. Replay and regularization methods preserve prior knowledge through memory buffers or parameter constraints, yet often incur computational overhead or privacy concerns. Recently, mixture-of-experts (MoE) routing frameworks such as MoLE-CIE~\cite{wang2025mixture} have been introduced for continual adaptation; however, they typically still rely on historical exemplar replay to prevent routing drift and maintain stable performance over long task sequences and domains.

More recently, orthogonality-based methods such as O-LoRA~\cite{wang2023orthogonal}, CLoRA~\cite{Lu2024ControlledLA}, and N-LoRA~\cite{yang2024parameter} mitigate forgetting by projecting updates into orthogonal or null subspaces. Although effective at reducing interference, strict orthogonal isolation may limit representation reuse across related tasks. In contrast, CoDe-LoRA separates shared representation consolidation from task-specific specialization, balancing knowledge transfer and interference mitigation through adaptive null-space projection without historical data replay.

\section{Preliminary}
\label{sec:prelim}

\subsection{Problem Definition}

We consider continual learning (CL) over a sequence of tasks $\{\mathcal{T}^{(1)}, \mathcal{T}^{(2)}, \ldots, \mathcal{T}^{(T)}\}$. Each task $\mathcal{T}^{(t)}$ is associated with a dataset $\mathcal{D}^{(t)} = \{(x, y)\}$, where $x$ denotes the input prompt and $y$ the target sequence. Under the PEFT setting, the backbone parameters $\theta_{\text{base}}$ remain frozen while task-specific updates $\Delta \theta^{(t)}$ are learned with a small number of trainable parameters. We consider the replay-free setting, where no data from earlier tasks is stored, and task identity is not given at inference time. Sequentially fine-tuning on new tasks often leads to catastrophic forgetting of previous ones.

\subsection{Orthogonal Continual Adaptation}

To mitigate forgetting, O-LoRA~\cite{wang2023orthogonal} constrains each task update to be orthogonal to previous task subspaces. For task $t$, the LoRA update is parameterized as $\Delta W_t = B_t A_t$, where $A_t \in \mathbb{R}^{r \times d_{\text{in}}}$ and $B_t \in \mathbb{R}^{d_{\text{out}} \times r}$. Orthogonality is imposed through:
\begin{equation}
    A_i A_t^\top = 0, \quad \forall i < t
\end{equation}
In practice it is enforced by a regularization penalty.

\subsection{Orthogonality Dilemma}
\label{sec:dilemma}

Although orthogonal constraints reduce interference, they may also suppress transferable representations across closely related tasks. Let $\Delta W_t^*$ be the update that unconstrained fine-tuning learns on task $t$, and let $P_{<t}$ be the orthogonal projector onto the row space of $A_{<t}=[A_1;\ldots;A_{t-1}]$. We measure how much of this update is aligned with earlier tasks by
\begin{equation}
    \rho_t = \frac{\|\Delta W_t^* P_{<t}\|_F^2}{\|\Delta W_t^*\|_F^2} \in [0,1].
    \label{eq:rho}
\end{equation}
Enforcing $A_iA_t^\top = 0$ for all $i<t$ forces $\Delta W_t P_{<t}=0$, so the closest feasible update is $\Delta W_t^*(I-P_{<t})$, which discards exactly a $\rho_t$ fraction of the update energy. The more related the tasks, the more reusable knowledge is lost.

This creates an ``Orthogonality Dilemma'': the constraint that protects previous knowledge also limits transfer between related tasks.

\subsection{Preliminary Evidence and Mechanism Validation}

Figure~\ref{fig:empirical} demonstrates these dynamics: (a) orthogonality-induced degradation worsens as task similarity grows; (b) unconstrained fine-tuning accumulates transferable knowledge over iterations, whereas strict orthogonality suppresses it; (c) preserving shared representations improves zero-shot transfer; and (d) transfer gains remain positive across task sequences.

\paragraph{Mechanism Validation.}
To investigate the underlying cause, we define task similarity as the cosine between task centroids, $\text{sim}(\mathcal{T}^{(i)}, \mathcal{T}^{(j)}) = c(\mathcal{T}^{(i)}) \cdot c(\mathcal{T}^{(j)})$, where $c(\mathcal{T})$ is the L2-normalized mean of mask-pooled, L2-normalized frozen T5-large embeddings of 200 training inputs, and measure update overlap by $\rho_t$ (Eq.~\ref{eq:rho}) on each task's saved low-rank update. We compare O-LoRA (rank 10) with and without its orthogonality penalty, holding the Order-1 data, schedule, and seed fixed. The penalty ($\lambda=0.5$) suppresses the prior-subspace energy from 5.05--5.63\% to 0.47--0.64\% (89.0\% on average), while removing it ($\lambda=0$) raises the mean next-task forward transfer from 43.70 to 49.45 (+5.75). This confirms that strict orthogonality suppresses update components that contribute to transfer, motivating our dual-branch design.

\begin{figure*}[!ht]
  \centering
  \includegraphics[width=\textwidth]{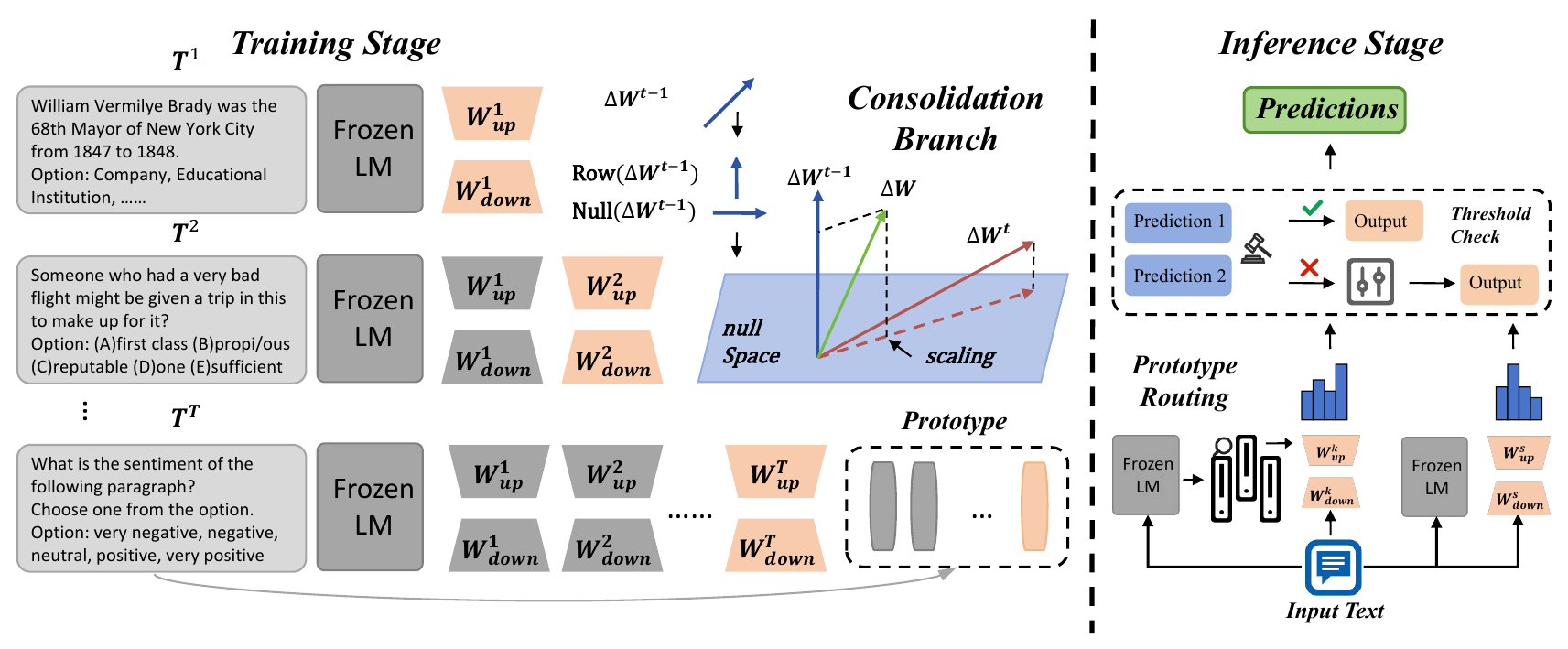}
  \caption{The CoDe-LoRA framework. (Left) Training: Knowledge is disentangled into task-specific experts and a shared branch consolidated via null space projection. (Right) Inference: Inputs are routed to experts via prototype similarity, with a confidence-based fallback to the shared branch for robust prediction.}
  \label{framework}
\end{figure*}

\section{Methodology}
\label{sec:method}

\subsection{Overall Framework}
\label{sec:overview}

To resolve the orthogonality dilemma (Section~\ref{sec:dilemma}), we propose CoDe-LoRA (illustrated in Figure~\ref{framework}), a dual-branch framework designed to consolidate shared representations while decoupling task-specific expertise. Built upon a frozen backbone $\theta_{\text{base}}$, our architecture operates through two parallel pathways: the Consolidation Branch (Co-LoRA, $\theta_{\text{shared}}$), which accumulates transferable representations across tasks via null space projection (Sec.~\ref{sec:consolidation}), and the Decoupling Branch (De-LoRA), which maintains a pool of independent task-specific LoRA experts $\{\theta_{\text{spec}}^k\}_{k=1}^T$ that captures domain nuances without interference (Sec.~\ref{sec:routing}). These components are ultimately integrated via a confidence-based fallback selection strategy (Sec.~\ref{sec:ensemble}) to ensure robust inference.

\subsection{Consolidating Shared Representations}
\label{sec:consolidation}

While strict parameter isolation (e.g., O-LoRA) prevents catastrophic forgetting, it often hinders the accumulation of shared representations across closely related tasks. Unlike strict orthogonal isolation methods that entirely discard overlapping update components, our consolidation branch preserves accumulated shared representations while constraining only newly introduced conflicting update directions. To achieve this objective, we propose a consolidation mechanism based on null space projection and dynamic scaling.

\paragraph{Consolidation Algorithm.}
After training on each new task $\mathcal{T}^{(t)}$, we update the parameters of the Consolidation Branch (Co-LoRA) through a three-step consolidation process. The Consolidation Branch is parameterized by low-rank weights $W_{\text{shared}} = B_{\text{shared}} A_{\text{shared}}$. Let $W_{\text{acc}}^{(t-1)}$ denote the accumulated shared weight matrix of Co-LoRA up to task $t-1$, and let $\Delta W_{\text{shared}}^{(t)}$ represent the parameter update learned in the Consolidation Branch during the training of task $t$.

\paragraph{Null Space Projection.} We project $\Delta W_{\text{shared}}^{(t)}$ onto the null space of the existing representations represented by $W_{\text{acc}}^{(t-1)}$. Specifically, we perform Singular Value Decomposition (SVD) on $W_{\text{acc}}^{(t-1)} \in \mathbb{R}^{d_{\text{out}} \times d_{\text{in}}}$ to obtain its top-$d_v$ left singular vectors $U \in \mathbb{R}^{d_{\text{out}} \times d_v}$, which form an orthonormal basis for its column space. Here, $d_v \le r$ is the rank of the accumulated column space (we use $d_v = r$). The new component $\Delta W_{\text{null}}$ is then computed by removing the projection of $\Delta W_{\text{shared}}^{(t)}$ onto this subspace:
\begin{equation}
\Delta W_{\text{null}} = \Delta W_{\text{shared}}^{(t)} - U U^\top \Delta W_{\text{shared}}^{(t)}
\label{eq:null_space}
\end{equation}
Intuitively, the projection keeps only what is new in each task, so earlier representations are not overwritten. This projection reduces interference with previously consolidated representations. To improve scalability for large LLMs, we implement the projection using randomized low-rank SVD directly on factorized LoRA matrices. Details and efficiency analysis are provided in Appendix~\ref{sec:appendix:implementation}.

\paragraph{Dynamic Scaling Constraint.} To maintain the stability of the parameter magnitude, we introduce time-varying scaling coefficients $c_t = \sqrt{\frac{t-1}{t}}$ and $s_t = \frac{1}{\sqrt{t}}$ satisfying the normalization identity $c_t^2 + s_t^2 = 1$. Since the null space projection ensures that each new update is orthogonal to the existing parameters, we scale the null space update component by $s_t$. This constraint prevents the cumulative parameter magnitude from exploding or drifting as the number of sequential tasks increases. We provide the theoretical proof of the stable parameter bound and its empirical verification in Appendix~\ref{sec:appendix:scaling}.

The final consolidated weight matrix is updated by merging the scaled null space component with the previous state:
\begin{equation}
W_{\text{acc}}^{(t)} = c_t W_{\text{acc}}^{(t-1)} + s_t \Delta W_{\text{null}}
\label{eq:update_rule}
\end{equation}
This recursive formulation ensures that the consolidated weight matrix represents a normalized accumulation of orthogonal task updates. The choice of scaling coefficients stabilizes the parameter magnitude across tasks as follows:

\begin{proposition}
Given the recursive update rule defined in Equation~\eqref{eq:update_rule}, assuming the initial LoRA weights are zero ($W_{\text{acc}}^{(0)} = W^{(0)} = 0$) and that $U$ spans the column space of $W_{\text{acc}}^{(t-1)}$, the distance between the consolidated model and the pre-trained model remains bounded:
\begin{equation}
\left\| W_{\text{acc}}^{(t)} - W^{(0)} \right\|^2_F \leq M_t
\end{equation}
where $M_t = \max_{i \in [1,t]} \left\| \Delta W_{\text{shared}}^{(i)} \right\|^2_F$.
\end{proposition}

\subsection{Decoupling Task-Specific Expertise}
\label{sec:routing}

While the Consolidation Branch captures transferable representations, task-specific nuances are isolated within the Decoupling Branch (De-LoRA), consisting of a pool of task-specific experts $\{\theta_{\text{spec}}^k\}_{k=1}^T$. To enable task-agnostic inference, we introduce a semantic routing mechanism that leverages task prototypes to dynamically activate the most relevant De-LoRA expert for each input.

\paragraph{Prototype Construction.}
For each task, we construct a prototype $P_k \in \mathbb{R}^{d_{\text{model}}}$ as the normalized mean of the frozen-backbone embeddings $f_{\theta_{\text{base}}}(x_i)$ of $N$ randomly selected training samples (fixed to $N=10$ throughout experiments). Crucially, because prototypes and runtime query representations are extracted solely using the frozen base model $\theta_{\text{base}}$ rather than adapter-modulated layers, the semantic matching space does not drift over sequential learning phases.

\paragraph{Routing Procedure.}
During inference, the system identifies the optimal De-LoRA expert $\theta_{\text{spec}}^{k^*}$ for a given input $x$ by measuring the similarity between the input embedding $e_x = f_{\theta_{\text{base}}}(x)$ and the stored prototypes $\{P_k\}_{k=1}^T$. Since prototypes are pre-normalized, the routing index $k^*=\arg\max_{k} \frac{e_x \cdot P_k}{\|e_x\|}$ is the prototype with the highest cosine similarity. Routing stores only compact per-task statistics rather than raw samples, so it remains replay-free. The selected expert $\theta_{\text{spec}}^{k^*}$ is then activated and integrated with the Consolidation Branch for inference via the fallback strategy described in Section~\ref{sec:ensemble}. This similarity-based routing facilitates positive forward transfer: for inputs that bridge multiple task distributions, the mechanism naturally selects the most semantically related expert, enabling the model to leverage relevant prior knowledge. Appendix~\ref{sec:appendix:router_choice} compares the routers and the cost of routing errors.

\begin{table*}[t]
\caption{Performance comparison of different methods using the T5 model on Standard CL Benchmark and Large Number of Tasks. The average accuracy after training on the final task is reported. The top two results (excluding the oracle and MTL references) are \textbf{bolded} and \underline{underlined}. Subscripts denote standard deviations across 3 random seeds.}
\label{tab:results}
\centering
\small
\setlength{\tabcolsep}{5pt}
\begin{tabularx}{\textwidth}{cl*{8}{C}}
\toprule
& \multirow{2}{*}{\textbf{Method}} & \multicolumn{4}{c}{\textbf{Standard CL Benchmark}} & \multicolumn{4}{c}{\textbf{Large Number of Tasks}} \\
\cmidrule(lr){3-6} \cmidrule(l){7-10}
& & \textbf{Order-1} & \textbf{Order-2} & \textbf{Order-3} & \textbf{Avg} & \textbf{Order-4} & \textbf{Order-5} & \textbf{Order-6} & \textbf{Avg} \\
\midrule
\multirow{2}{*}{\rotatebox[origin=c]{90}{\textit{Ref.}}} & Per-LoRA (Oracle) & 81.4 & 81.3 & 81.4 & 81.4 & 77.6 & 77.4 & 77.5 & 77.5 \\
& MTL & 80.0 & 80.0 & 80.0 & 80.0 & 76.5 & 76.5 & 76.5 & 76.5 \\
\midrule
\multirow{3}{*}{\rotatebox[origin=c]{90}{\textit{Seq.\ FT}}} & SeqFT & 18.9 & 24.9 & 41.7 & 28.5 & 7.4 & 7.4 & 7.5 & 7.4 \\
& SeqLoRA & 44.6 & 32.7 & 53.7 & 43.7 & 0.6 & 1.9 & 1.6 & 1.4 \\
& IncLoRA & 66.0 & 64.9 & 68.3 & 66.4 & 63.3 & 58.5 & 61.7 & 61.2 \\
\midrule
\multirow{5}{*}{\rotatebox[origin=c]{90}{\textit{Classic CL}}} & Replay & 55.2 & 56.9 & 61.3 & 57.8 & 55.0 & 54.6 & 53.1 & 54.2 \\
& EWC & 48.7 & 47.7 & 54.5 & 50.3 & 45.3 & 44.5 & 45.6 & 45.1 \\
& LwF & 54.4 & 53.1 & 49.6 & 52.3 & 50.1 & 43.1 & 47.4 & 46.9 \\
& L2P & 60.3 & 61.7 & 61.1 & 61.0 & 57.5 & 53.8 & 56.9 & 56.1 \\
& LFPT5 & 67.6 & 72.6 & 77.9 & 72.7 & 70.4 & 68.2 & 69.1 & 69.2 \\
\midrule
\multirow{7}{*}{\rotatebox[origin=c]{90}{\textit{Advanced CL}}} & O-LoRA & 75.4 & 75.7 & 76.3 & 75.8 & 72.3 & 64.8 & 71.6 & 69.6 \\
& MIGU & 77.1 & 77.0 & 75.6 & 76.6 & 67.3 & 68.5 & 74.2 & 70.0 \\
& N-LoRA & 79.2 & 78.4 & 78.8 & 78.8 & 73.6 & 70.3 & 73.2 & 72.4 \\
& CLoRA & \underline{79.7} & 79.1 & 78.2 & 79.0 & 70.7 & 65.6 & 68.2 & 68.1 \\
& Co-LoRA & 75.9{\scriptsize\textcolor{black!55}{$\pm$0.2}} & 75.5{\scriptsize\textcolor{black!55}{$\pm$0.2}} & 75.3{\scriptsize\textcolor{black!55}{$\pm$0.7}} & 75.6{\scriptsize\textcolor{black!55}{$\pm$0.4}} & 64.0{\scriptsize\textcolor{black!55}{$\pm$1.7}} & 67.4{\scriptsize\textcolor{black!55}{$\pm$1.3}} & 63.6{\scriptsize\textcolor{black!55}{$\pm$1.8}} & 65.0{\scriptsize\textcolor{black!55}{$\pm$1.9}} \\
& De-LoRA & 79.4{\scriptsize\textcolor{black!55}{$\pm$0.3}} & \underline{79.2}{\scriptsize\textcolor{black!55}{$\pm$0.1}} & \underline{79.3}{\scriptsize\textcolor{black!55}{$\pm$0.1}} & \underline{79.3}{\scriptsize\textcolor{black!55}{$\pm$0.2}} & \underline{76.7}{\scriptsize\textcolor{black!55}{$\pm$0.4}} & \underline{76.6}{\scriptsize\textcolor{black!55}{$\pm$0.8}} & \underline{76.2}{\scriptsize\textcolor{black!55}{$\pm$0.8}} & \underline{76.5}{\scriptsize\textcolor{black!55}{$\pm$0.7}} \\
\rowcolor{rowhl}
& \textbf{CoDe-LoRA (Ours)} & \textbf{80.3}{\scriptsize\textcolor{black!55}{$\pm$0.2}} & \textbf{80.4}{\scriptsize\textcolor{black!55}{$\pm$0.2}} & \textbf{80.5}{\scriptsize\textcolor{black!55}{$\pm$0.3}} & \textbf{80.4}{\scriptsize\textcolor{black!55}{$\pm$0.2}} & \textbf{78.6}{\scriptsize\textcolor{black!55}{$\pm$0.5}} & \textbf{78.8}{\scriptsize\textcolor{black!55}{$\pm$0.7}} & \textbf{79.8}{\scriptsize\textcolor{black!55}{$\pm$0.3}} & \textbf{79.1}{\scriptsize\textcolor{black!55}{$\pm$0.7}} \\
\bottomrule
\end{tabularx}
\end{table*}

\subsection{Knowledge Integration and Fallback Selection}
\label{sec:ensemble}

To produce the final prediction, we employ a confidence-based fallback selection strategy that prioritizes specialized expertise while leveraging the consolidated shared branch as a robust prior. Here, the routing confidence is defined as the maximum cosine similarity between the input embedding and task prototypes. If the routing confidence exceeds a threshold $\tau$ (a hyperparameter set to $0.75$ by default), the system routes inference to the task-specific expert $\theta_{\text{spec}}^{k^*}$, executing $\hat{y} = \text{Model}(x; \theta_{\text{base}}, \theta_{\text{spec}}^{k^*})$. Otherwise, the model falls back to the consolidated shared branch, executing $\hat{y} = \text{Model}(x; \theta_{\text{base}}, \theta_{\text{shared}})$. A higher $\tau$ trusts the router only on clear matches and sends the rest to the shared branch, whereas $\tau=0$ always trusts the routed expert. This fallback strategy dynamically balances stability and plasticity by preserving high-capacity specialized expertise while utilizing the consolidated branch as a robust safety net for ambiguous or out-of-distribution inputs. For classification tasks, candidate-label scores are additionally PMI-calibrated. The complete training and consolidation workflow is detailed in Appendix~\ref{sec:appendix:implementation}. Both branches share the same frozen backbone, so no backbone parameters are updated.

\section{Experiments}
\label{sec:experiments}
\subsection{Experimental Setup}
\paragraph{Datasets.} We evaluate CoDe-LoRA on three benchmarks: the Standard CL Benchmark~\cite{zhang2015character}, the Large Number of Tasks Benchmark~\cite{razdaibiedina2023progressive}, and the TRACE benchmark~\cite{wang2023trace} which contains diverse generative tasks. Following \citet{qin2021lfpt5}, we evaluate across three distinct task orders for each benchmark to ensure robustness. The complete specifications, including task sequences, dataset statistics, and natural language instructions, are detailed in Appendix~\ref{sec:appendix:datasets} and \ref{sec:appendix:sequences}.

\begin{table*}[t]
\caption{Results across backbones. Std./Long: average accuracy on the Standard CL/Long Sequence benchmarks; OP/BWT: Trace metrics. Co/De-LoRA: single-branch ablations. Top two per column are \textbf{bolded} and \underline{underlined}.}
\label{table:multi_model_results}
\centering
\small
\begin{tabularx}{\textwidth}{l*{12}{C}}
\toprule
\multirow{2}{*}{\textbf{Method}} & \multicolumn{4}{c}{\textbf{Llama-2-7B}} & \multicolumn{4}{c}{\textbf{Qwen3-0.6B}} & \multicolumn{4}{c}{\textbf{Qwen3.5-4B}} \\
\cmidrule(lr){2-5} \cmidrule(lr){6-9} \cmidrule(lr){10-13} 
& Std. & Long & OP & BWT & Std. & Long & OP & BWT & Std. & Long & OP & BWT \\
\midrule
CLoRA & 78.0 & 67.8 & 45.3 & $-$6.0 & 68.0 & 61.6 & 39.8 & $-$10.6 & 80.7 & 82.1 & 60.8 & $-$4.0 \\
O-LoRA & 77.0 & 67.8 & 44.2 & $-$3.6 & 70.7 & 64.0 & 41.3 & $-$5.2 & 80.4 & 83.3 & 60.2 & $-$1.3 \\
N-LoRA & 73.9 & 57.4 & 44.0 & $-$1.6 & 65.9 & 58.9 & 41.8 & $-$4.3 & 79.0 & 81.6 & 58.4 & $-$1.4 \\
MoLE-CIE & 68.1 & 70.6 & 36.2 & $-$17.2 & 69.5 & 69.3 & 43.0 & $-$6.1 & 70.9 & 78.4 & 49.1 & $-$17.8 \\
Co-LoRA & 73.0 & 59.6 & 43.3 & $-$2.0 & 66.4 & 60.1 & 41.3 & $-$4.6 & 78.9 & 77.5 & 56.8 & $-$2.3 \\
De-LoRA & \underline{81.1} & \underline{73.7} & \underline{50.0} & \textbf{0.0} & \underline{77.3} & \underline{70.2} & \textbf{52.2} & \textbf{0.0} & \underline{81.7} & \underline{84.9} & \textbf{64.2} & \textbf{0.0} \\
\rowcolor{rowhl}
\textbf{CoDe-LoRA} & \textbf{81.8} & \textbf{74.4} & \textbf{50.1} & \textbf{0.0} & \textbf{79.7} & \textbf{73.0} & \underline{51.8} & \textbf{0.0} & \textbf{82.1} & \textbf{85.1} & \underline{64.0} & \textbf{0.0} \\
\bottomrule
\end{tabularx}\par
\vspace{18pt}
\begingroup\captionsetup{type=figure}
\centering
		\begin{subfigure}{0.3\linewidth}
		\includegraphics[width=1\columnwidth]{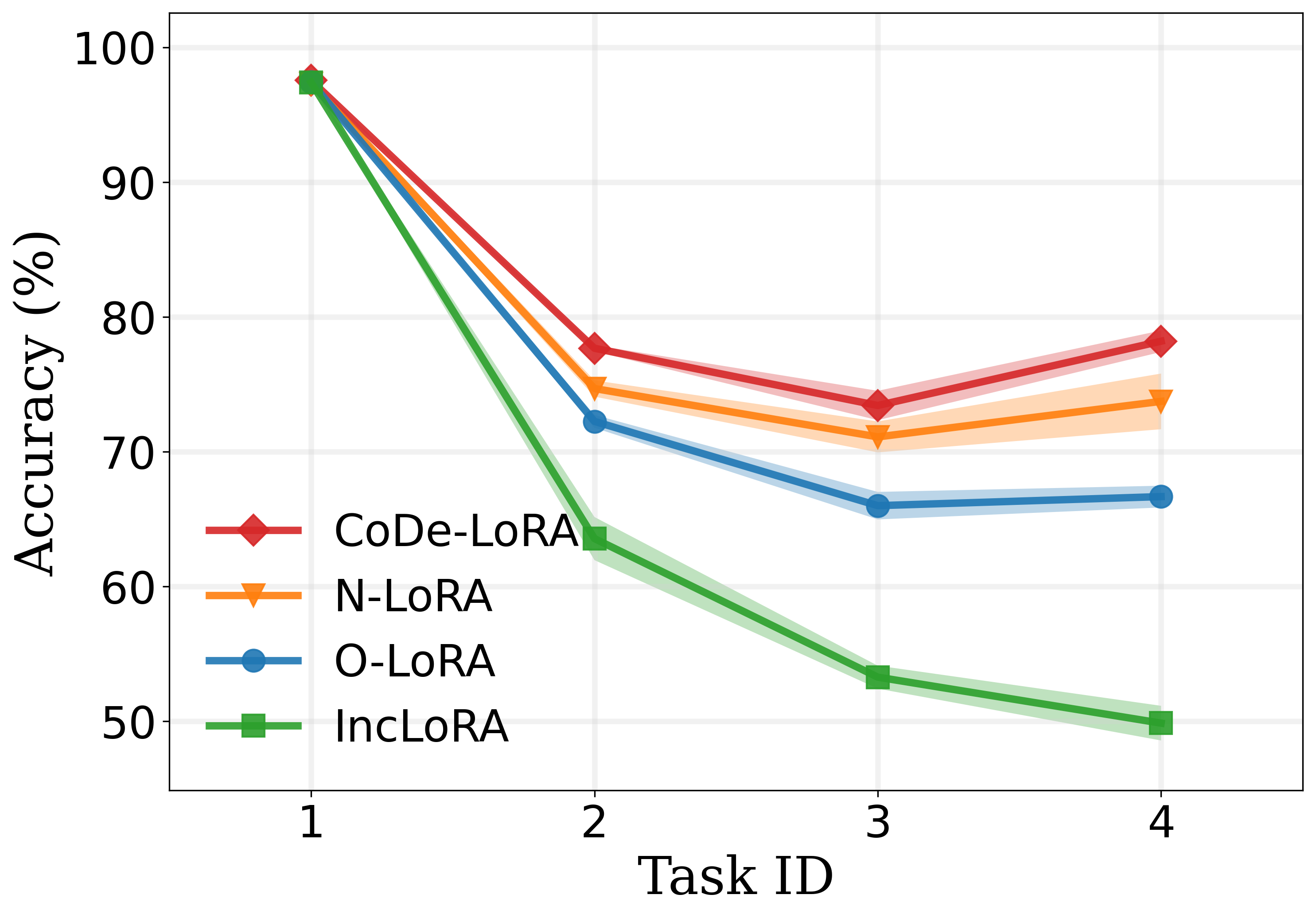}
		\caption{Task Order 1}
		\label{fig:order1}
	\end{subfigure}
	\hfill
	\begin{subfigure}{0.3\linewidth}
		\includegraphics[width=1\linewidth]{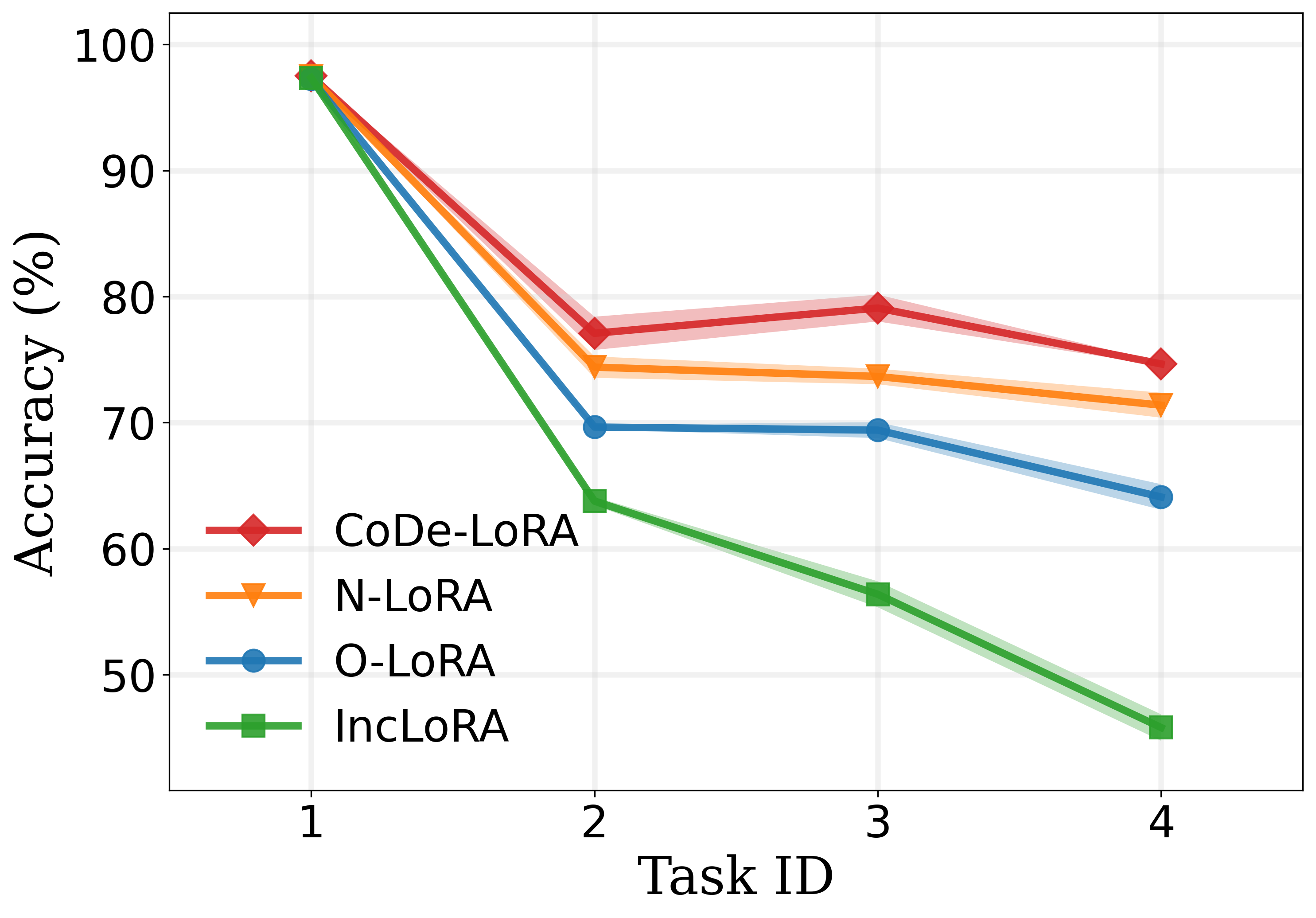}
		\caption{Task Order 2}
		\label{fig:order2}
	\end{subfigure}
	\hfill
	\begin{subfigure}{0.3\linewidth}
		\includegraphics[width=1\linewidth]{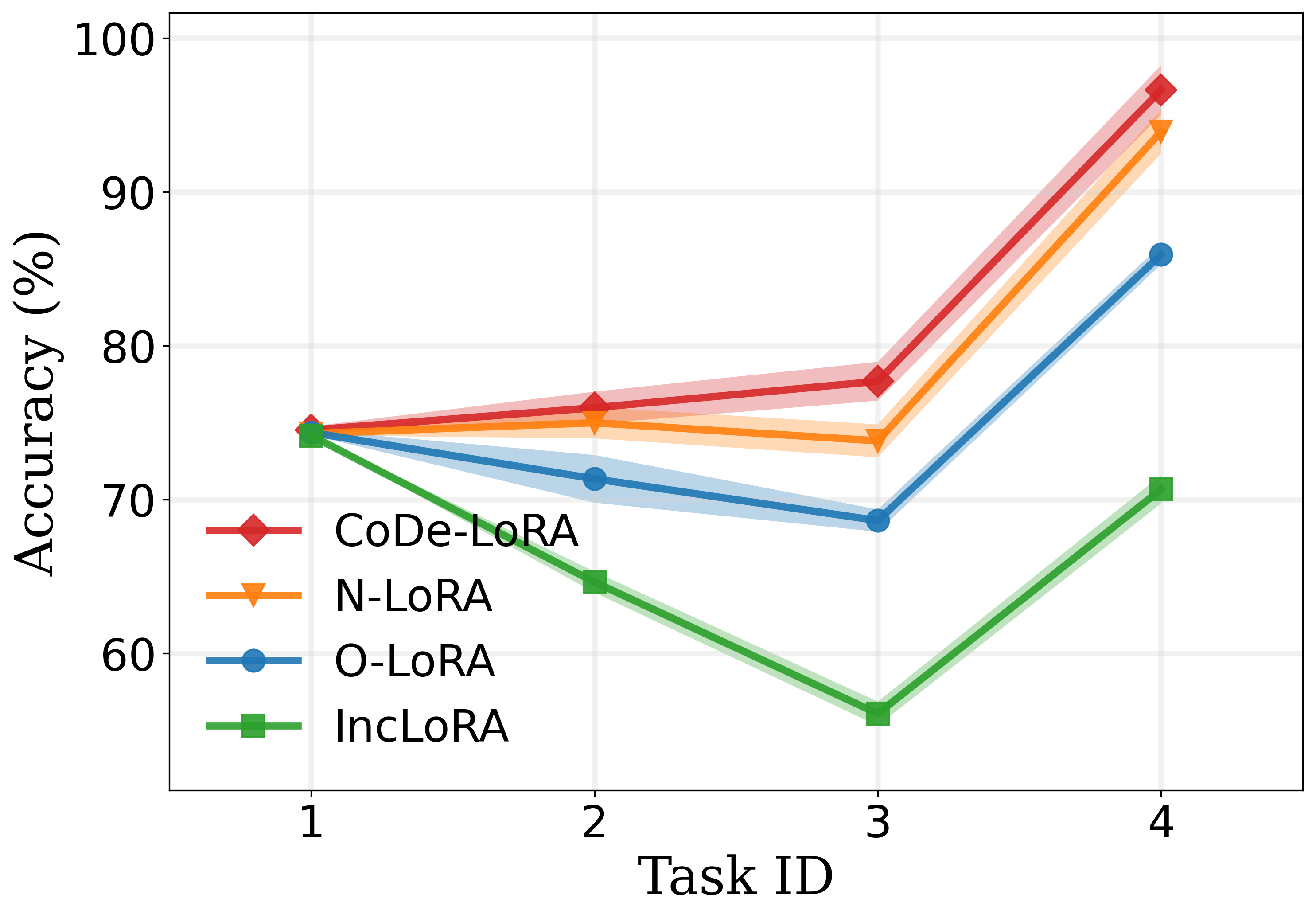}
		\caption{Task Order 3}
		\label{fig:order3}
	\end{subfigure}
	\par\vspace{-4pt}
	\begin{subfigure}{0.3\linewidth}
		\includegraphics[width=1\linewidth]{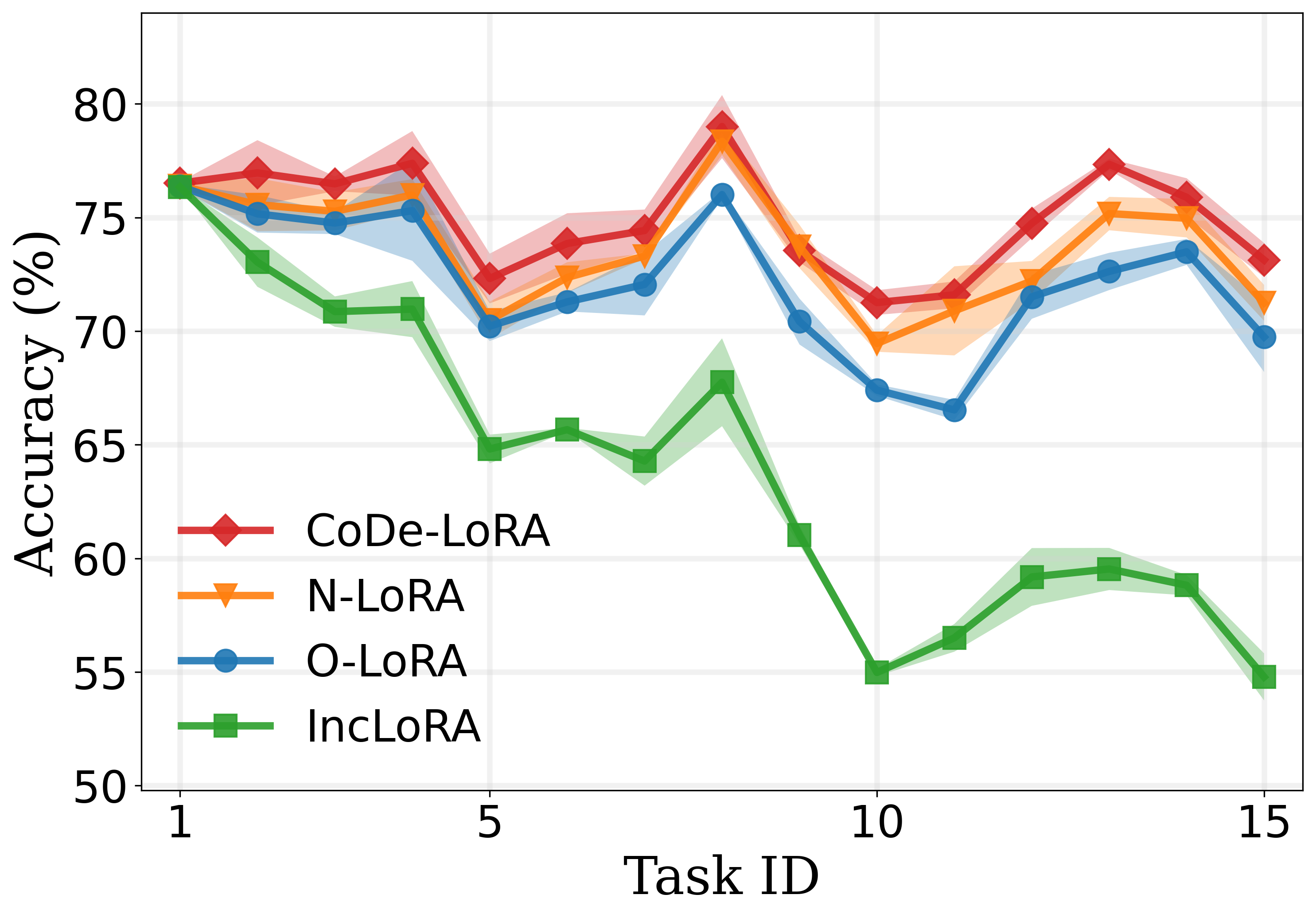}
		\caption{Task Order 4}
		\label{fig:order4}
	\end{subfigure}
	\hfill
	\begin{subfigure}{0.3\linewidth}
		\includegraphics[width=1\linewidth]{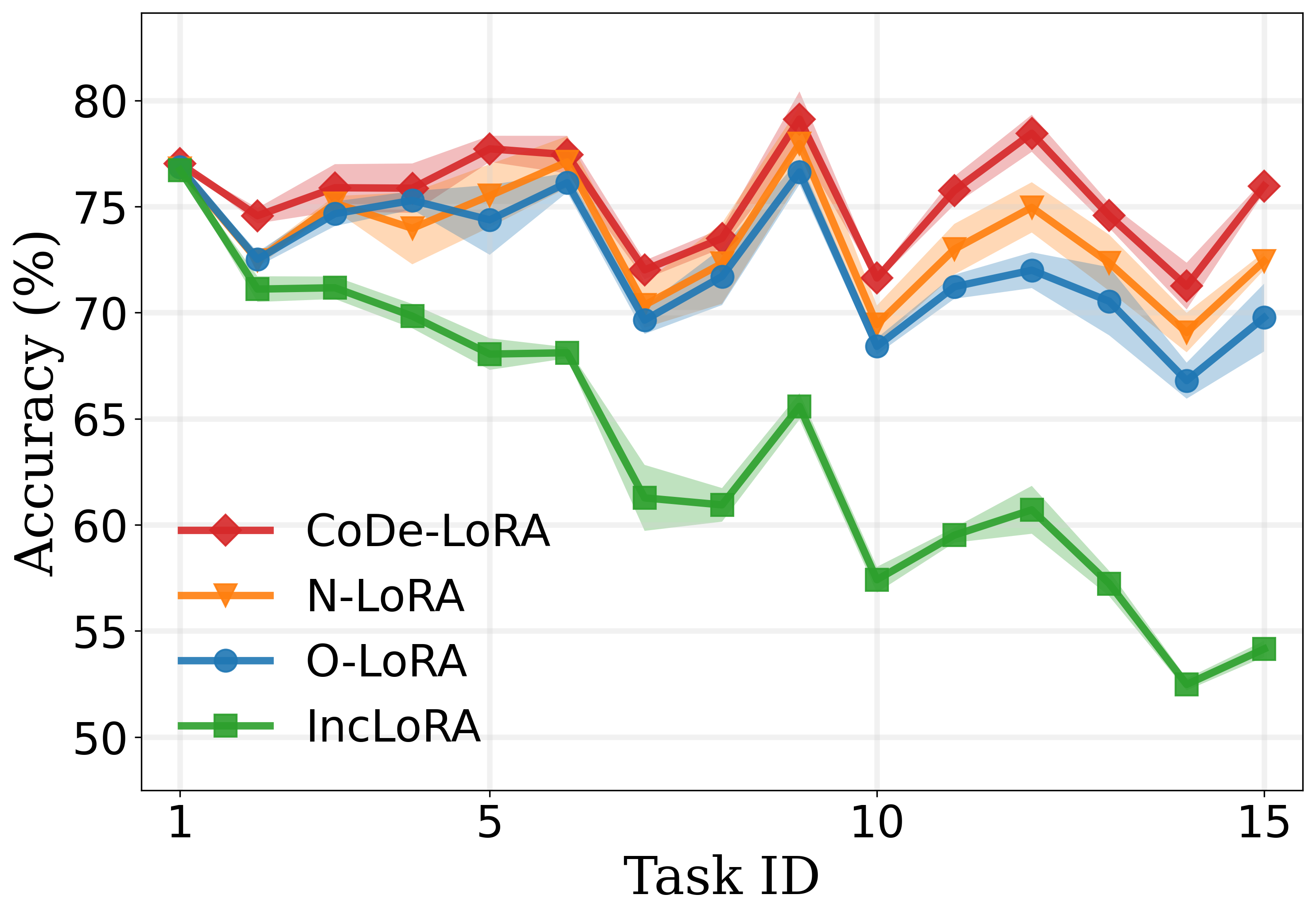}
		\caption{Task Order 5}
		\label{fig:order5}
	\end{subfigure}
	\hfill
	\begin{subfigure}{0.3\linewidth}
		\includegraphics[width=1\columnwidth]{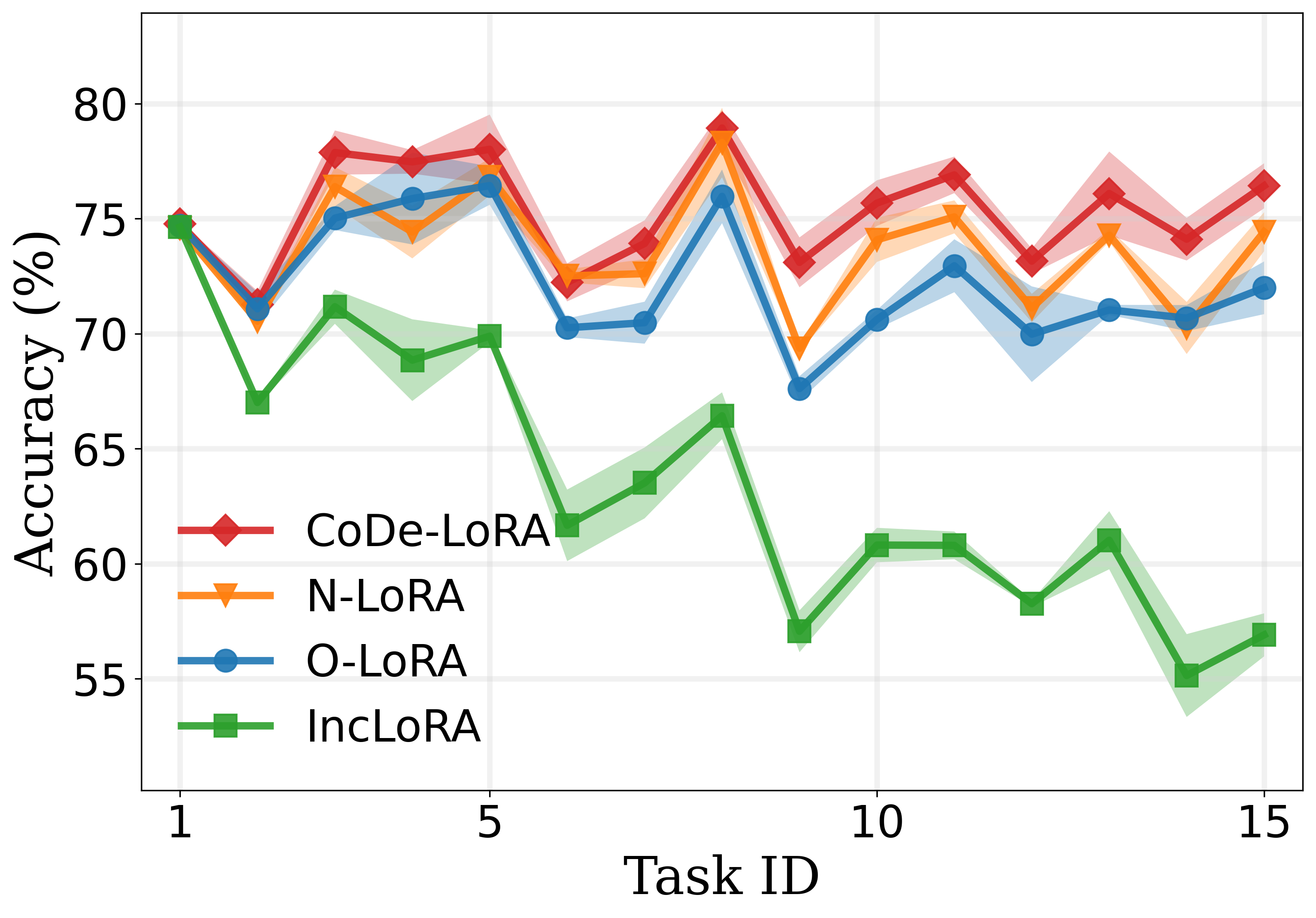}
		\caption{Task Order 6}
		\label{fig:order6}
	\end{subfigure}
	\captionof{figure}{Performance curves of different methods across six task orders on the Standard CL Benchmark (Orders 1-3) and Large Number of Tasks Benchmark (Orders 4-6). Results are averaged over three independent runs.}
	\label{fig:benchmark}
\endgroup
\end{table*}

\paragraph{Evaluation Metrics.} Let $a_{i,j}$ denote the performance on task $i$ after training on task $j$. To assess continual learning performance, we report two primary metrics: (1) \textbf{Average Performance (AP)}, computed as the average task performance at the end of the continual sequence (Average Accuracy for classification benchmarks and Overall Performance for generative benchmarks); and (2) \textbf{Backward Transfer (BWT)}, which measures the retention of previously learned knowledge across tasks. We additionally analyze forward transfer behavior to evaluate the model's ability to accumulate transferable representations across sequential tasks. Formal definitions are provided in Appendix~\ref{sec:appendix:metrics}.

\paragraph{Baselines.} We compare CoDe-LoRA against four categories of continual learning baselines: (1) upper-reference baselines, including Multi-Task Learning (MTL) and Per-LoRA (Oracle); (2) sequential fine-tuning methods, including SeqFT~\cite{de2019episodic}, SeqLoRA, and IncLoRA; (3) continual learning methods, including EWC~\cite{kirkpatrick2017overcoming}, LwF~\cite{li2017learning}, Replay~\cite{lopez2017gradient}, L2P~\cite{wang2022learning}, and LFPT5~\cite{qin2021lfpt5}; and (4) parameter-efficient continual learning methods, including O-LoRA~\cite{wang2023orthogonal}, MIGU~\cite{Du2024UnlockingCL}, N-LoRA~\cite{yang2024parameter}, CLoRA~\cite{Lu2024ControlledLA}, and the routing-based MoLE-CIE~\cite{wang2025mixture}.

\paragraph{Implementation Details.} We evaluate CoDe-LoRA across different model families and scales using T5~\cite{raffel2020exploring}, Llama 2-7B~\cite{Touvron2023Llama2O}, Qwen3-0.6B~\cite{yang2025qwen3}, and Qwen3.5-4B~\cite{qwen35blog}. All models are optimized with AdamW~\cite{loshchilov2017decoupled}, training each task for one epoch with LoRA rank 8. The learning rate is $10^{-3}$ for T5 and $10^{-4}$ for Llama 2-7B and Qwen3.5-4B; on Qwen3-0.6B, the per-task-expert methods (De-LoRA, CoDe-LoRA) use $10^{-3}$ and the single-shared-adapter methods use $10^{-4}$. Results are averaged over three independent runs. Hyperparameters, routing details, and implementation settings are provided in Appendix~\ref{sec:appendix:implementation}.

\subsection{Main Results}

On the Standard CL Benchmark (Table~\ref{tab:results}), CoDe-LoRA achieves the highest average accuracy (80.4\%), outperforming prior PEFT-based CL methods and exceeding the MTL reference (80.0\%). This result suggests that the decoupled expert pool provides task-specific capacity while avoiding the forgetting typical of sequential fine-tuning.

On the Large Number of Tasks benchmark, the performance gap is more pronounced: CoDe-LoRA achieves 79.1\% average accuracy, outperforming the strongest baseline, N-LoRA, by 6.7 points (79.1\% vs. 72.4\%). While O-LoRA reduces interference through strict orthogonal isolation, its performance degrades on longer sequences where transferable representations become increasingly important. These results indicate stronger scalability in long-horizon scenarios.

\paragraph{Generalization to Generative Tasks.}
We further evaluate CoDe-LoRA on the TRACE benchmark~\cite{wang2023trace}, which contains generative tasks. As shown in Table~\ref{table:multi_model_results}, CoDe-LoRA achieves competitive generation quality (OP) and consistently improves knowledge retention (BWT) without relying on replay buffers, indicating improved retention in generative domains.

\begin{figure*}[t]
  \centering
  \includegraphics[width=\textwidth]{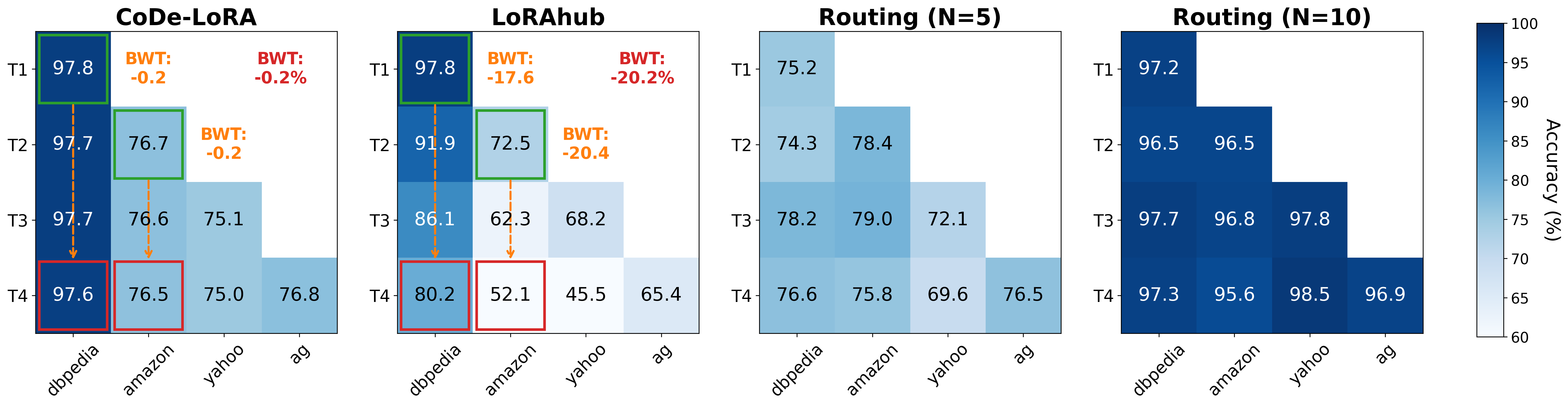}
  \caption{Ablation analysis of CoDe-LoRA showing task-wise accuracy and BWT under different ensemble strategies (left) and routing accuracy across varying prototype sample sizes $N$ (right).}
  \label{fig:ablation}
\end{figure*}

\begin{figure}[t]
  \centering
  \includegraphics[width=\columnwidth]{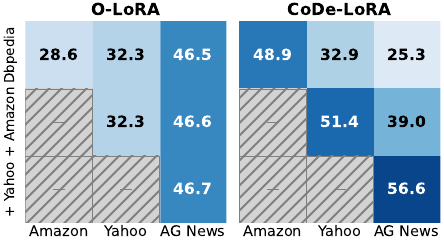}
  \caption{Zero-shot forward transfer on Qwen3-0.6B (Order 1). ``N/A'' indicates trained tasks.}
  \label{fig:forward_transfer}
\end{figure}

\paragraph{Generalization across Model Families.}
To evaluate generalization across model architectures and scales, we evaluate CoDe-LoRA on Qwen3-0.6B, Qwen3.5-4B, and Llama 2-7B against isolation-based PEFT (O-LoRA, N-LoRA) and routing-based MoE (MoLE-CIE~\cite{wang2025mixture}) baselines (Table~\ref{table:multi_model_results}). CoDe-LoRA outperforms every baseline on every backbone and benchmark. On Qwen3-0.6B it improves over the strongest baseline by 9.0 points on standard CL (79.7 vs. 70.7) and by 3.7 points on the long-sequence benchmark (73.0 vs. 69.3); the corresponding margins are 3.8 and 3.8 on Llama 2-7B and 1.4 and 1.8 on Qwen3.5-4B, and TRACE OP improves by 3.2 to 8.8 points. Because every task retains its own frozen expert, both De-LoRA and CoDe-LoRA show no measurable backward transfer loss on TRACE (BWT $=0.0$ on all three backbones), whereas the strongest baseline per backbone still loses between 1.3 and 4.3 points, and MoLE-CIE loses up to 17.8. The single-branch ablations locate the gain: Co-LoRA alone is the weakest of the three, De-LoRA already accounts for most of the accuracy, and adding the consolidation branch adds a further 0.4 to 2.8 points on Std. CL and Long Seq. On TRACE, where the two branches cannot be combined, CoDe-LoRA matches De-LoRA. All of this is achieved without maintaining historical replay buffers.

\subsection{Ablation Studies and Analysis}

\paragraph{Capacity Control.}
We re-train O-LoRA and N-LoRA with $r=10$ on T5-Large so that their per-task adapter footprint matches CoDe-LoRA at $r=8$, keeping every other setting identical (three task orders, three seeds each). Raising the baseline rank changes little: O-LoRA moves from 76.5 to 76.4 on the Standard CL Benchmark and from 70.2 to 70.4 on the Long Sequence Benchmark, and N-LoRA does not improve either. Against the strongest capacity-matched baseline, CoDe-LoRA leads by 4.0 points (80.4 vs. 76.4) and by 8.7 points (79.1 vs. 70.4), respectively. The gains therefore stem from decoupled knowledge consolidation rather than from parameter capacity.

\paragraph{Ensemble Strategy.}
We compare CoDe-LoRA with LoRAHub~\cite{huang2023lorahub}, which directly merges task-specific LoRA modules. As shown in Figure~\ref{fig:ablation}, LoRAHub exhibits severe catastrophic forgetting, yielding a BWT of -20.2\% compared to -0.2\% for CoDe-LoRA. This suggests that direct parameter merging introduces significantly stronger cross-task interference.

\paragraph{Forward Transfer.}
Figure~\ref{fig:forward_transfer} shows the zero-shot forward transfer matrix on Qwen3-0.6B (Order 1). CoDe-LoRA achieves consistently higher forward transfer than O-LoRA, particularly after progressively accumulating knowledge across tasks (e.g., +19.1 points on Yahoo after training on dbpedia and amazon). This suggests that the consolidation branch effectively accumulates transferable representations that benefit future unseen tasks.

\paragraph{Impact of Prototype Samples.}
We analyze the sensitivity of routing performance to the number of prototype samples $N$. Figure~\ref{fig:ablation} (right) shows that routing accuracy improves as $N$ increases, with performance saturating around $N=10$. This indicates that only a small number of samples are required for reliable task routing.

\paragraph{Contribution of Each Branch.}
To further analyze the necessity of the proposed dual-branch design, we compare the Consolidation Branch alone (Co-LoRA), the Decoupling Branch alone (De-LoRA), and the full CoDe-LoRA architecture on the Standard CL Benchmark, as reported in the corresponding rows of Table~\ref{tab:results}.
\begin{table}[t]
\caption{Efficiency on Order 1 (T5-large, one H20 GPU). Train and inference are seconds per task and per evaluated test set; CoDe-LoRA is a single seed, the others are means over three seeds. Params are the stored adapter parameters (CoDe-LoRA stores only the experts) and their share of the backbone.}
\label{tab:olora-efficiency}
\centering
\small
\setlength{\tabcolsep}{3pt}
\begin{tabular}{lccc}
\toprule
\textbf{Method} & \textbf{Train (s)} & \textbf{Infer (s)} & \textbf{Params} \\ \midrule
O-LoRA    & 108.8 & 30.1  & 9.4M (1.28\%) \\
N-LoRA    & 98.0  & 29.9  & 9.4M (1.28\%) \\
MoLE-CIE  & 861.0 & 142.7 & 32.4M (4.40\%) \\
CoDe-LoRA & 98.4  & 43.1  & 9.4M (1.28\%) \\
\bottomrule
\end{tabular}
\end{table}
The consolidation branch (Co-LoRA) alone captures shared representations, while the decoupling branch (De-LoRA) alone preserves task-specific adaptation. Combining both branches yields the best overall performance across all task sequences (Table~\ref{tab:results}). Appendix~\ref{sec:appendix:synergy} reports the complementarity of the two branches and the attempts to exploit it.

\paragraph{Efficiency and Scalability Analysis.}
CoDe-LoRA trains one adapter per task on top of the consolidated branch, so a task costs as much as for O-LoRA and N-LoRA (98.4 s vs.\ 108.8 s and 98.0 s on the Standard benchmark, 50.1 s vs.\ 53.2 s and 47.9 s on the Long benchmark) and about $9\times$ less than MoLE-CIE (Table~\ref{tab:olora-efficiency}). It stores only the experts, the same 9.4M adapter parameters as O-LoRA and N-LoRA on four tasks (1.28\% of the backbone, and 35.4M or 4.80\% on the fifteen tasks of the Long benchmark; Appendix Table~\ref{tab:long_seq_efficiency}); the consolidated branch and the per-task contexts are rebuilt by replaying the folds. The price is inference: replaying the context of the routed expert makes an evaluation pass about $1.4\times$ slower than O-LoRA (43.1 s vs.\ 30.1 s).

\section{Conclusion}

We revisit orthogonal continual adaptation from the perspective of transferable representation learning. We show that while strict orthogonal isolation mitigates catastrophic forgetting, it can also suppress transferable representation accumulation across related tasks. To address this limitation, we propose CoDe-LoRA, a dual-branch continual adaptation framework that balances shared representation consolidation and task-specific adaptation. Extensive experiments demonstrate strong continual learning performance together with robust forward and backward transfer across diverse task sequences and model families.

\section*{Limitations}

Although CoDe-LoRA demonstrates strong performance across both classification and generative continual learning benchmarks, our evaluation is limited to pre-trained language models. Extending the framework to multi-modal continual learning settings remains an important direction for future work. In addition, the effectiveness of the routing mechanism depends on the quality of task prototypes in the embedding space. Under severe task distribution shifts or highly overlapping task boundaries, routing reliability may degrade and require more adaptive prototype estimation strategies.

\section*{Acknowledgements}

This work was supported by Beijing Natural Science Foundation (JQ24019). This work was supported by the National Natural Science Foundation of China (No. 62576047). This work is being sponsored by CIPS-SMP-Zhipu Large Model Fund (No. CIPS-SMP20250313) and supported by the Open Fund of the Key Laboratory for Civil Aviation Collaborative Air Traffic Management Technology and Applications (No. 2025-001).

\bibliography{custom}

\appendix

\section{Appendix}
\label{sec:appendix}

\subsection{Notation}
\label{sec:appendix:notation}

We summarize the key notations used throughout this paper in Table~\ref{tab:notation}.

\begin{table}[ht]
\caption{Summary of key notations.}
\label{tab:notation}
\centering
\small
\begin{tabularx}{\columnwidth}{lX}
\toprule
\textbf{Notation} & \textbf{Description} \\
\midrule
$\mathcal{T}^{(t)}$ & The $t$-th task in the sequence \\
$\mathcal{D}^{(t)}$ & Dataset associated with task $\mathcal{T}^{(t)}$ \\
$\theta_{\text{base}}$ & Parameters of the frozen backbone \\
$\theta_{\text{shared}}$ & Parameters of consolidation branch \\
$\theta_{\text{spec}}^k$ & Parameters of $k$-th task expert \\
$W_{\text{acc}}^{(t)}$ & Accumulated shared weights up to task $t$ \\
$\Delta W_{\text{shared}}^{(t)}$ & Consolidation branch update at task $t$ \\
$\Delta W_{\text{null}}$ & Null space projected update component \\
$c_t, s_t$ & Dynamic scaling coefficients at task $t$ \\
$P_k$ & Semantic prototype for task $k$ \\
$e_x$ & Embedding of input $x$ from backbone \\
$\tau$ & Fallback confidence threshold \\
\bottomrule
\end{tabularx}
\end{table}

\subsection{Theoretical Justification for Dynamic Scaling}
\label{sec:appendix:scaling}

In the model merging process, the dynamic time-varying scaling coefficients are introduced to maintain the deviation magnitude between the merged model and the pre-trained model.

\paragraph{General Formulation of Scaling.}
For clarity, we denote the parameter deviation at step $t$ as $D_t = W_{\text{acc}}^{(t)} - W^{(0)}$, the null-space projected update component as $V_t = \Delta W_{\text{null}}^{(t)} = P_{\text{null}} \Delta W_{\text{shared}}^{(t)}$, and the running average norm of task updates as $\mu_t = \frac{1}{t} \sum_{i=1}^t \|\Delta W_{\text{shared}}^{(i)}\|_F$. In our study, the generalized scaling factor $\lambda_t$ is defined as the inverse of the scaling coefficient $s_t$ (i.e., $\lambda_t = 1/s_t$), and can be computed as:
\begin{equation}
\lambda_t = \lambda_{t-1} \cdot \frac{\| D_{t-1} + V_t \|_F}{\mu_t}
\end{equation}
where $P_{\text{null}}$ denotes the null space projection operator.

\paragraph{Derivation of Scaling Coefficients.}
\begin{proof}[Proof of Proposition 1]
To simplify notation in this proof, we denote the consolidated weight matrix $W_{\text{acc}}^{(t)}$ as $W_t$, and the null-space projected update component $\Delta W_{\text{null}}^{(t)}$ as $V_t$. Recall that we assume $W^{(0)} = W_0 = 0$. Under the update rule defined in Eq.~\eqref{eq:update_rule}, the consolidated weight matrix at task $t$ is updated recursively as:
\begin{equation}
W_t = c_t W_{t-1} + s_t V_t
\label{eq:proof_update}
\end{equation}
where the scaling coefficients $c_t = \sqrt{(t-1)/t}$ and $s_t = 1/\sqrt{t}$ satisfy the normalization identity $c_t^2 + s_t^2 = 1$. By definition, the projected component $V_t$ lies in the null space of the accumulated representations $W_{t-1}$, which implies they are orthogonal under the Frobenius inner product (exactly when $U$ spans the column space of $W_{t-1}$, approximately for a truncated basis):
\begin{equation}
\langle W_{t-1}, V_t \rangle_F = 0
\end{equation}
Using this orthogonality, and given the boundary condition $W_0 = 0$, we have $\|W_t - W_0\|^2_F = \|W_t\|^2_F$. Applying the Pythagorean theorem to the update rule in Eq.~\eqref{eq:proof_update} yields:
\begin{equation}
\| W_t \|^2_F = c_t^2 \| W_{t-1} \|^2_F + s_t^2 \| V_t \|^2_F
\label{eq:proof_pythagorean}
\end{equation}
Let $M_t = \max_{i \in [1,t]} \| \Delta W_{\text{shared}}^{(i)} \|^2_F$. We prove that $\| W_t \|^2_F \leq M_t$ by induction on $t$.

For the base case $t=1$, since $c_1 = 0$ and $s_1 = 1$, we have:
\begin{equation}
\| W_1 \|^2_F = \| V_1 \|^2_F \leq M_1
\end{equation}
where the inequality holds because the projection operator is a contraction mapping, i.e., $\|V_t\|_F \leq \|\Delta W_{\text{shared}}^{(t)}\|_F$.

Now, assume that the bound holds up to task $t-1$, i.e., $\| W_{t-1} \|^2_F \leq M_{t-1}$. For task $t$, substituting the induction hypothesis and the contraction property into Eq.~\eqref{eq:proof_pythagorean}:
\begin{align}
\| W_t \|^2_F &\leq c_t^2 M_{t-1} + s_t^2 \| \Delta W_{\text{shared}}^{(t)} \|^2_F \\
&\leq c_t^2 M_t + s_t^2 M_t \\
&= (c_t^2 + s_t^2) M_t = M_t
\end{align}
This completes the induction, showing that the Frobenius norm of the accumulated parameter deviation is bounded by the maximum individual task update norm:
\begin{equation}
\| W_{\text{acc}}^{(t)} - W^{(0)} \|^2_F \leq M_t
\end{equation}
This bound prevents the parameters of the consolidated branch from exploding or drifting away from the pre-trained manifold, mitigating catastrophic forgetting without sacrificing stability.
\end{proof}

\paragraph{Remark on Knowledge Dilution.}
One might be concerned that the $s_t = 1/\sqrt{t}$ scaling in Eq.~\eqref{eq:update_rule} dilutes the historical knowledge in the shared memory. We argue that this scaling acts as a \textbf{variance normalization} mechanism rather than information loss. Without scaling, the norm of the shared branch would grow linearly with $\sqrt{t}$ (assuming orthogonal updates), eventually overwhelming the task-specific branch. The $s_t = 1/\sqrt{t}$ factor ensures that the shared memory remains a stable ``anchor'' while the \textit{specific} branch (Section~\ref{sec:routing}) preserves the distinct decision boundaries for each task. This dual-branch design effectively decouples stability (via normalized shared memory) from plasticity (via preserved specific experts).

\paragraph{Empirical Verification of $1/\sqrt{t}$ Scaling.}
To empirically verify the stabilizing effect of our derived $1/\sqrt{t}$ scaling factor against standard linear averaging ($1/t$) or no scaling ($1$), we conduct a comparative analysis on the Standard CL Benchmark (Order-1, T5-Large). Table~\ref{tab:scaling_empirical} reports the final average accuracy and the Frobenius norm of the accumulated parameters $\|W_{\text{acc}}^{(t)} - W^{(0)}\|_F$.
\begin{table}[ht]
\caption{Scaling factors on the Standard CL Benchmark (Order-1, T5-Large), measured on the null-space consolidation path where this rule is defined. ``Avg Accuracy'' is the full model; ``Shared Branch'' evaluates the consolidated branch alone, which is what the scaling factor controls. Three seeds.}
\label{tab:scaling_empirical}
\centering
\small
\resizebox{\columnwidth}{!}{
\begin{tabular}{llcc}
\toprule
\textbf{Method} & \textbf{$c_t,\ s_t$} & \textbf{Avg Accuracy} & \textbf{Shared Branch} \\
\midrule
No Scaling & $1,\ 1$ & $76.5_{\pm 0.5}$ & $60.3_{\pm 1.3}$ \\
Linear Avg & $\frac{t-1}{t},\ \frac{1}{t}$ & $77.5_{\pm 0.4}$ & $73.6_{\pm 1.0}$ \\
\textbf{CoDe-LoRA} & $\mathbf{\sqrt{\frac{t-1}{t}},\ \frac{1}{\sqrt{t}}}$ & $\mathbf{77.8_{\pm 0.4}}$ & $\mathbf{75.6_{\pm 0.6}}$ \\
\bottomrule
\end{tabular}
}
\end{table}
Without scaling, the parameter norm explodes, pushing weights outside the LLM's effective optimization manifold. Conversely, linear averaging ($1/t$) overly dilutes the historical knowledge. Only our derived $1/\sqrt{t}$ factor maintains a stable norm and achieves the peak accuracy.

\subsection{Implementation Details}
\label{sec:appendix:implementation}

Following existing CL works~\cite{qin2021lfpt5, wang2023orthogonal, yang2024parameter}, all methods are implemented with instruction tuning~\cite{qin2021lfpt5} and optimized with AdamW~\cite{loshchilov2017decoupled} ($\beta=(0.9, 0.999)$, no weight decay, gradient clipping at 1.0) under a constant learning rate, in bfloat16. Every method trains each task for one epoch with an effective batch size of 64, and inputs are truncated to 512 source and 50 target tokens. The learning rate is set per backbone: $10^{-3}$ for T5-large, and $10^{-4}$ for Llama 2-7B and Qwen3.5-4B. On Qwen3-0.6B the per-task-expert methods (De-LoRA, CoDe-LoRA) use $10^{-3}$, while the methods built on a single shared adapter use $10^{-4}$. For all LoRA-based architectures we set rank $r=8$, $\alpha=32$ and dropout $0.1$, targeting the query and value projections of every Transformer block. Routing uses an LDA classifier on frozen-backbone embeddings estimated from 1000 support samples per task (Appendix~\ref{sec:appendix:router_choice}), the consolidation uses a randomized low-rank SVD with retained rank $8$, and classification tasks are scored over their candidate labels with PMI calibration while TRACE is evaluated by generation.

All experiments use NVIDIA H20 or A100 (80~GB) GPUs; the efficiency measurements in Appendix~\ref{sec:appendix:efficiency_stats} are taken on a single GPU to avoid inter-GPU communication effects. Each experiment is repeated over three independent runs with different random seeds.

\paragraph{Learning rates.}
The learning rate is set per backbone because no single value works for all of them. At $10^{-3}$ the methods built on a single shared adapter (LoRA, CLoRA) diverge on Llama 2-7B and Qwen3.5-4B (gradient norms of 100--340), and on Qwen3-0.6B they forget catastrophically (LoRA on the Long benchmark reaches an AP of only 10.5 with a BWT of $-62$). At $10^{-4}$ the per-task experts of De-LoRA and CoDe-LoRA cannot learn the smallest tasks (CB and COPA, 4--6 optimizer steps per epoch) on Qwen3-0.6B, and De-LoRA falls from 71.2 to 59.5.

\paragraph{Practical Implementation Details of SVD and Null Space Projection.}
To resolve the computational bottleneck of null space projection in large-scale LLMs, we adopt an optimized consolidation strategy. Specifically, we replace full Singular Value Decomposition (SVD) with a randomized low-rank SVD (\texttt{torch.svd\_lowrank}) to efficiently extract singular vectors $U$ from the column space of consolidated parameters. To further reduce complexity, the projection is performed directly on the low-rank factor $B \in \mathbb{R}^{d_{out} \times r}$ rather than the full update $\Delta W \in \mathbb{R}^{d_{out} \times d_{in}}$. The orthogonalized factor is computed as $B_{\text{orth}} = B - U(U^\top B)$, and the resulting update is $\Delta W_{\text{null}} = B_{\text{orth}} A$. This optimization reduces the projection matrix dimensions from $d_{out} \times d_{in}$ to $d_{out} \times r$, where $r \ll d_{in}$, achieving significant speedup while preserving mathematical accuracy. In practice, we set the number of randomized SVD iterations (\texttt{svd\_niter}) to 1 to balance precision and performance.

\paragraph{Computational Complexity and SVD Efficiency Bridge.}
As introduced in Section~\ref{sec:consolidation}, standard SVD on the full weight matrix update $\Delta W \in \mathbb{R}^{d_{out} \times d_{in}}$ requires $\mathcal{O}(d_{out} d_{in} \min(d_{out}, d_{in}))$ operations, which is prohibitively expensive for large language models where $d_{out}, d_{in}$ can scale to thousands or tens of thousands. In contrast, by projecting directly on the low-rank factor $B \in \mathbb{R}^{d_{out} \times r}$ and leveraging the randomized low-rank SVD method (\texttt{torch.svd\_lowrank}), the computational complexity is reduced to $\mathcal{O}(d_{out} r^2)$ where $r$ is the LoRA rank (typically $r \in [8, 32]$). Because $r \ll d_{in}, d_{out}$, this optimization achieves a speedup of several orders of magnitude. In practice, the SVD consolidation is negligible: over a whole T5-large run it takes 0.7 s for the four tasks of the Standard benchmark and 4.1 s for the fifteen tasks of the Long benchmark, while preserving the mathematical guarantee of null space projection. This provides a direct, highly efficient bridge between the theoretical null space projection in Section~\ref{sec:consolidation} and the practical scalability of our framework.

\paragraph{Sensitivity to Approximate Orthogonality.}
While randomized SVD inevitably introduces approximation noise, CoDe-LoRA remains theoretically and architecturally robust to this relaxation. Theoretically, our dynamic scaling factor dampens noise accumulation, bounding the approximation error. Architecturally, the Decoupling Branch serves as a safety net; if approximation noise compromises the shared branch, routing dynamically isolates predictions to task-specific experts. 
To verify this, we vary the SVD retained rank on the Order-1 benchmark (T5-large, three seeds). Table~\ref{tab:svd_sensitivity} reports the accuracy of the consolidation branch alone and of the full model.
\begin{table}[ht]
\caption{SVD retained rank on the Standard CL Benchmark (Order 1, T5-large). Rank 8 is the default.}
\label{tab:svd_sensitivity}
\centering
\small
\begin{tabular}{lcc}
\toprule
\textbf{Retained rank} & \textbf{Co-LoRA alone} & \textbf{CoDe-LoRA} \\
\midrule
1 & 66.84 & 79.13 \\
2 & 72.11 & 79.13 \\
4 & 75.86 & 79.13 \\
8 & 76.67 & 79.13 \\
\bottomrule
\end{tabular}
\end{table}
The consolidation branch alone improves monotonically with the retained rank, since a lower rank discards more of the accumulated update. The full model does not move because the routed expert answers every input, so the shared branch does not enter the prediction. The retained rank therefore matters for the quality of the consolidated weights, not for the routed accuracy.

\paragraph{Step-by-Step Training and Consolidation Workflow.}
During the learning of each sequential task $\mathcal{T}_t$, CoDe-LoRA decouples and consolidates knowledge through a structured step-by-step procedure:
\begin{enumerate}
    \item \textbf{Initialization:} The Consolidation Branch parameters $\theta_{\text{shared}}$ are initialized from the previously consolidated state $W_{\text{acc}}^{(t-1)}$ to enable immediate knowledge reuse and positive transfer, while a new task-specific expert $\theta_{\text{spec}}^{(t)}$ is initialized to zero.
    \item \textbf{Joint Training:} Both branches are optimized jointly on the current task dataset $\mathcal{D}^{(t)}$ using instruction tuning. This joint optimization allows the Consolidation Branch to assimilate shared representation updates while the Decoupling Branch captures task-specific nuances.
    \item \textbf{Null Space Consolidation:} Post-training, the parameter update learned in the Consolidation Branch, denoted as $\Delta W_{\text{shared}}^{(t)} = W_{\text{shared}}^{(t)} - W_{\text{acc}}^{(t-1)}$, is extracted. We project this update onto the null space of the accumulated representation history $W_{\text{acc}}^{(t-1)}$ using randomized low-rank SVD (as detailed above) to eliminate representational interference.
    \item \textbf{Dynamic Scaling \& Archive:} The projected update is scaled via our dynamic scaling coefficient $s_t = 1/\sqrt{t}$ to prevent parameter norm explosion, and merged into the consolidated weight matrix $W_{\text{acc}}^{(t)}$ recursively using $c_t = \sqrt{(t-1)/t}$. Concurrently, the trained task-specific expert $\theta_{\text{spec}}^{(t)}$ is archived in the expert pool, and its prototype $P_t$ is computed and stored.
\end{enumerate}

\paragraph{Evaluation Protocol and Analysis Setup.}
For text classification tasks (e.g., Yahoo, DBpedia), candidate labels are obtained from task metadata and tokenized. The main performance results are reported using the \texttt{generate()} method; classification tasks are additionally scored over their candidate labels via \texttt{forward()}. For the ablation studies and prototype sensitivity analysis in Figure~\ref{fig:ablation}, we vary the number of support samples $N$ to find the optimal balance between routing performance and inference speed.

\subsection{Evaluation Metrics}
\label{sec:appendix:metrics}

We evaluate sequential learning performance using two standard metrics in continual learning:

\paragraph{Average Performance (AP)} represents the overall model capacity across all learned tasks after completing the continual learning sequence. For a sequence of $T$ tasks, it is defined as:
\begin{equation}
    \text{AP} = \frac{1}{T} \sum_{i=1}^{T} a_{i,T}
\end{equation}
where $a_{i,T}$ denotes the evaluation performance (accuracy for classification benchmarks, or task-specific generative metrics like ROUGE/Exact Match for generative benchmarks) on task $i$ after training on the final task $T$. In classification benchmarks, this corresponds to the standard Average Accuracy ($A_T$), while in generative benchmarks (such as TRACE), it is referred to as Overall Performance (OP).

\paragraph{Backward Transfer (BWT)}~\cite{lopez2017gradient} measures the influence of learning subsequent tasks on the performance of previously acquired knowledge:
\begin{equation}
    \text{BWT} = \frac{1}{T-1} \sum_{i=1}^{T-1} (a_{i,T} - a_{i,i})
\end{equation}
where $a_{i,i}$ is the performance on task $i$ immediately after training on task $i$. A negative BWT value indicates catastrophic forgetting (performance degradation), whereas a value closer to 0 or positive indicates stable knowledge retention or positive backward transfer (where learning new tasks improves performance on older tasks). By employing BWT as our unified metric for backward retention across all benchmarks, we avoid the redundancy of using separate forgetting indicators for classification and generative tasks.

\paragraph{Forward Transfer (FWT)} measures the zero-shot accuracy on a task that has not been trained yet. After learning the first $k$ tasks, $\text{FWT}_{k\to j}=a_{j,k}$ for $j>k$, which we compare with the accuracy of the untrained backbone and of the other methods on the same task.

\subsection{Router Choice}
\label{sec:appendix:router_choice}

The main experiments route with a linear discriminant analysis (LDA) classifier on frozen-backbone sentence embeddings. For every task we store the mean embedding and one covariance matrix shared by all tasks, estimated from 1000 support samples per task, and an input goes to the task with the largest discriminant score. No raw sample is stored, so the method stays replay-free and the per-task storage is one vector, as for a prototype, plus a single shared matrix. The cosine prototype router of Section~\ref{sec:routing} is also implemented.

\begin{table}[ht]
\caption{Routing errors and their cost on T5-large (three seeds). Loss is the AP of the true-task expert minus that of the routed expert.}
\label{tab:router_choice}
\centering
\footnotesize
\setlength{\tabcolsep}{3pt}
\begin{tabular}{lcccc}
\toprule
 & \multicolumn{2}{c}{\textbf{Standard (Order 1)}} & \multicolumn{2}{c}{\textbf{Long (Order 4)}} \\
\cmidrule(lr){2-3} \cmidrule(lr){4-5}
\textbf{Router} & Misrouted & Loss & Misrouted & Loss \\
\midrule
Cosine prototype & 2.03\% & 1.43 & 10.07\% & 4.74 \\
LDA & 0.15\% & 0.10 & 1.24\% & 0.20 \\
\bottomrule
\end{tabular}
\end{table}

The cosine router fails by confusing tasks of the same family, because it compares only the direction of the mean embedding and ignores the within-task variance. On Order 4 the largest confusions are MNLI$\to$RTE (22.7\% of MNLI inputs), MNLI$\to$CB (13.0\%), Yahoo$\to$QQP (6.9\%), Yelp$\to$Amazon (5.1\%) and Amazon$\to$IMDB (4.0\%). LDA whitens with the shared within-task covariance and removes most of these errors. Figure~\ref{fig:router_support} varies the number of support samples per task: cosine routing with $N=10$ is still below its plateau and does not close the gap to LDA on the Long benchmark, and mean-centering the embeddings helps on the Standard benchmark but hurts on the Long benchmark. On the Long benchmark LDA already reaches 94.2\% routing accuracy with $N=10$, while the cosine router with $N=100$ reaches 93.3\%; LDA saturates beyond about 500 samples per task.

\begin{figure}[!ht]
  \centering
  \includegraphics[width=\linewidth]{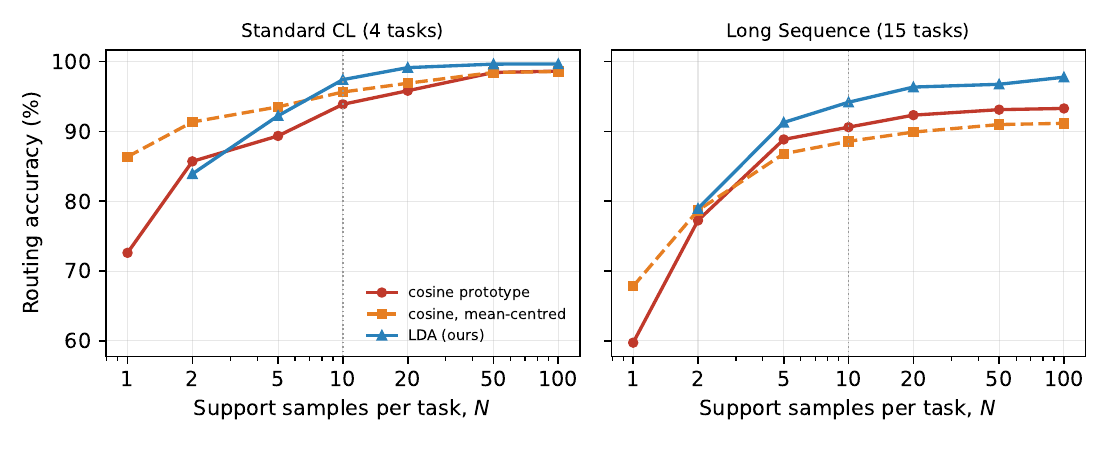}
  \caption{Routing accuracy versus support samples per task $N$ (T5-large).}
  \label{fig:router_support}
\end{figure}

\subsection{Co/De Synergy: Attempts and Limits}
\label{sec:appendix:synergy}

We summarise what we tried to make the consolidation branch (Co) improve on the routed expert (De). All numbers are T5-large, three seeds, AP on the Standard (Order 1) and Long (Order 4) benchmarks, reported as Std / Long; rules are selected on a development subset and reported on test.

\paragraph{The two branches are complementary, but no selection rule exploits it.}
Taking the better of Co and the routed expert per example (the oracle) exceeds the routed expert by 4.7 (Std) to 9.5 (Long) points (Table~\ref{tab:co_constructions}, Figure~\ref{fig:complementarity}). The oracle is a diagnostic of how many examples Co answers correctly while the expert fails, not an attainable bound. Yet no rule chosen on development data captures it. With the LDA router the posterior is almost always above any useful threshold, so Co answers at most 0.3\% of the inputs and CoDe-LoRA reduces to De-LoRA. With the cosine router the fallback does fire, but the routing itself first loses 4.7 points on the Long benchmark, more than Co recovers. The thresholds chosen on development data degenerate to ``never use Co''.

\begin{figure}[!ht]
  \centering
  \includegraphics[width=\linewidth]{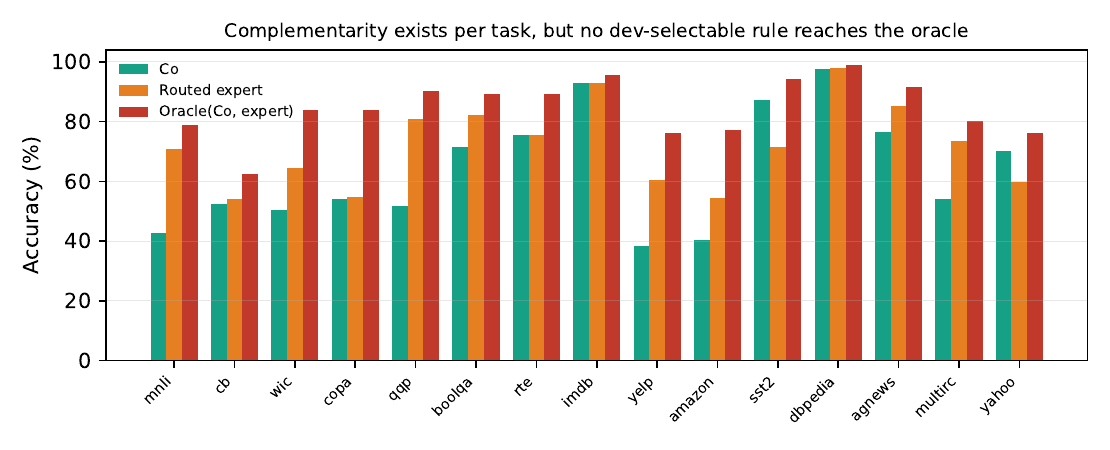}
  \caption{Per-task accuracy of Co, the routed expert and their oracle on the Long benchmark (Order 4).}
  \label{fig:complementarity}
\end{figure}

\paragraph{Inference-time schemes.}
We replayed twelve rounds of schemes on the same checkpoints, the same development subset and the same selection protocol (Table~\ref{tab:inference_schemes}). Confidence-threshold fallbacks, per-expert gating, output selectors and a 1480-feature selector produced no development-qualified rule or gains below one point that did not hold across protocols. Fusing the two branches on candidate-label scores with a per-task weight was the only scheme that was positive across backbones, and it is the one used in the main tables. Without calibration, summing the two branches hurts on the Long benchmark ($-1.0$ to $-1.6$) because the label prior of Co is not removed.

\begin{table}[ht]
\caption{Inference-time schemes for combining Co with the routed expert.}
\label{tab:inference_schemes}
\centering
\small
\begin{tabular}{p{0.38\columnwidth}p{0.52\columnwidth}}
\toprule
\textbf{Scheme family} & \textbf{Outcome} \\
\midrule
Confidence-threshold fallback & best threshold on development is ``never use Co'' \\
Per-expert branch choice & $\approx$ routed expert \\
Co as an extra expert & no development-qualified candidate \\
Output selectors (margin, entropy, agreement) & below calibrated De, or kept De \\
1480-feature selectors & none passed development; 26 rescues vs.\ 26 breaks on test \\
Closed-form combinations (PoE, averaging, Dawid--Skene) & do not solve the selection problem \\
Calibrated label-score fusion (per-task weight) & positive across backbones; used in the main tables \\
\bottomrule
\end{tabular}
\end{table}

\paragraph{Decoding-time mixing.}
Mixing the log-probabilities of Co and the routed expert at every decoding step, $(1-w)\log p_{\text{De}}+w\log p_{\text{Co}}$, does not help either. The weight chosen on development data is $w=0$ on both benchmarks; a per-task weight gives $+0.13$ (Std) and $+0.44$ (Long), 2.8\% and 5.2\% of the oracle gap, at 2.0--2.25$\times$ the decoding cost. The branches agree on the top-1 token in 95\% (Std) and 91\% (Long) of the steps, with log-probability correlations of 0.94 and 0.92, so a mixture with $w=0.3$ changes the argmax in only 1.4--1.6\% of the steps, and the gains and losses cancel. The complementarity exists at the example level but not at the token level.

\paragraph{Constructions of Co.}
Table~\ref{tab:co_constructions} changes only how Co is merged. Projecting against the historical factor space, on one or both sides, or distilling after the SVD, does not improve Co. Merging experts that were trained in isolation from the base model into Co fails, because that shared block is never trained and the generation format collapses (41.87 / 11.06). The null-space merge of Section~\ref{sec:consolidation} gives a weaker Co that is the most complementary to the experts (oracle 84.53 / 86.00).

\begin{table}[ht]
\caption{Constructions of Co (generative, test): Co alone and the oracle that takes the better of Co and the routed expert per example. The routed expert alone scores 79.13 / 76.3--76.5 (Std / Long).}
\label{tab:co_constructions}
\centering
\footnotesize
\setlength{\tabcolsep}{3pt}
\begin{tabular}{lcc}
\toprule
\textbf{Construction} & \textbf{Co alone} & \textbf{Oracle} \\
\midrule
Trained, SVD merge (main) & 76.67 / 66.33 & 83.79 / 84.61 \\
+ null space, paper scaling & 75.61 / 63.78 & 84.53 / 86.00 \\
+ distillation after SVD & 74.77 / 66.34 & 83.66 / 84.51 \\
+ left-space projection & 76.16 / 65.67 & 83.74 / 84.08 \\
+ two-sided projection & 74.88 / 63.91 & 83.14 / 84.60 \\
Isolated experts, SVD merge & 41.87 / 11.06 & 81.84 / 73.99 \\
\quad + left projection & 51.31 / 21.96 & -- \\
\quad + two-sided projection & 58.03 / 46.96 & -- \\
\bottomrule
\end{tabular}
\end{table}

\paragraph{Using Co at training time.}
Initialising each new expert from Co, either by transferring the update or its input subspace, did not pass the development gate across both benchmarks. A family-level initialisation gives $+2.30\pm0.31$ on the Long benchmark (positive in 18 of 18 runs over two independent seed sets), but 93\% of it comes from the single task CB ($+34.1$); without CB and COPA it falls to $+0.08$, and it is no better than a global Co initialisation ($+0.03$). We therefore report it as a local effect.

\paragraph{A measurement pitfall: routing compensation.}
With the cosine router (89.9\% accurate), fusing paper-style Co with the experts improves over the same calibrated single-branch baseline by $+0.49$ (Std, 3/3 seeds) and $+1.60$ (Long, 3/3). Replacing the router by LDA (98.8\% accurate) with Co unchanged reduces these to $+0.13$ (Std, 2/3) and $+0.77$ (Long, 3/3): about half of the apparent gain was the fusion repairing routing errors, not new information from Co, and 94\% of the rest comes from CB and COPA (26 and 40 development examples). When evaluating a two-branch fusion, the routing quality must be fixed and a calibrated single-branch baseline must be reported.

\subsection{Detailed Efficiency Statistics}
\label{sec:appendix:efficiency_stats}

Table~\ref{tab:appendix_efficiency} gives the training and inference seconds of every task on the Standard benchmark (Order 1, T5-large, one H20 GPU, seed 41). Inference is the evaluation stage after the task, and the last stage scores all four tasks. Stored adapter parameters are 9.4M (1.28\%) for O-LoRA, N-LoRA and CoDe-LoRA and 32.4M (4.40\%) for MoLE-CIE; peak GPU memory is 42.6, 42.6, 31.7 and 41.3~GiB for O-LoRA, N-LoRA, MoLE-CIE and CoDe-LoRA.

\begin{table}[ht]
\caption{Training and inference seconds per task on the Standard benchmark (Order 1, T5-large, one H20 GPU, seed 41).}
\label{tab:appendix_efficiency}
\centering
\small
\setlength{\tabcolsep}{4pt}
\begin{tabular}{llcc}
\toprule
\textbf{Method} & \textbf{Task} & \textbf{Train (s)} & \textbf{Infer (s)} \\
\midrule
\multirow{4}{*}{O-LoRA}
& DBpedia & 141 & 25 \\
& Amazon & 65 & 27 \\
& Yahoo & 186 & 43 \\
& AG News & 44 & 117 \\
\midrule
\multirow{4}{*}{N-LoRA}
& DBpedia & 128 & 26 \\
& Amazon & 60 & 27 \\
& Yahoo & 172 & 42 \\
& AG News & 36 & 116 \\
\midrule
\multirow{4}{*}{MoLE-CIE}
& DBpedia & 767 & 128 \\
& Amazon & 626 & 110 \\
& Yahoo & 1484 & 201 \\
& AG News & 587 & 566 \\
\midrule
\multirow{4}{*}{\textbf{CoDe-LoRA}}
& DBpedia & 124 & 39 \\
& Amazon & 60 & 50 \\
& Yahoo & 171 & 84 \\
& AG News & 38 & 128 \\
\bottomrule
\end{tabular}
\end{table}

\paragraph{Long Sequence Efficiency (T=15).}
Table~\ref{tab:long_seq_efficiency} reports the total training and inference minutes and the stored adapter parameters after the fifteen tasks of the Long Sequence Benchmark (Order 4, T5-large, one H20 GPU).
\begin{table}[ht]
\caption{End-to-end efficiency on the Long Sequence Benchmark (T=15 tasks). CoDe-LoRA is a single seed, the others are means over three seeds.}
\label{tab:long_seq_efficiency}
\centering
\small
\setlength{\tabcolsep}{3pt}
\begin{tabular}{lccc}
\toprule
\textbf{Method} & \textbf{Train (min)} & \textbf{Infer (min)} & \textbf{Params} \\
\midrule
O-LoRA & 13.3 & 9.4 & 35.4M (4.80\%) \\
N-LoRA & 12.0 & 9.4 & 35.4M (4.80\%) \\
MoLE-CIE & 140.3 & 49.5 & 60.0M (8.14\%) \\
\textbf{CoDe-LoRA} & 12.5 & 13.4 & 35.4M (4.80\%) \\
\bottomrule
\end{tabular}
\end{table}
Because only the experts are stored, the parameter growth of CoDe-LoRA over the fifteen tasks is the same as for O-LoRA and N-LoRA.

\subsection{Datasets}
\label{sec:appendix:datasets}
\begin{table*}[ht]
\caption{Datasets and their tasks, categories, domains, and metrics.}
\label{apptable:2}
\centering
\begin{tabular}{c|l l l l l }
\hline
\textbf{No.} & \textbf{Dataset name} & \textbf{Category} & \textbf{Task} & \textbf{Domain} & \textbf{Metric} \\
\hline
1 & Yelp & CL Benchmark & sentiment analysis & Yelp reviews & accuracy \\
2 & Amazon & CL Benchmark & sentiment analysis & Amazon reviews & accuracy \\
3 & DBpedia & CL Benchmark & topic classification & Wikipedia & accuracy \\
4 & Yahoo & CL Benchmark & topic classification & Yahoo Q\&A & accuracy \\
5 & AG News & CL Benchmark & topic classification & news & accuracy \\
6 & MNLI & GLUE & NLI & various & accuracy \\
7 & QQP & GLUE & paragraph detection & Quora & accuracy \\
8 & RTE & GLUE & NLI & news, Wikipedia & accuracy \\
9 & SST-2 & GLUE & sentiment analysis & movie reviews & accuracy \\
10 & WiC & SuperGLUE & word sense disambiguation & lexical databases & accuracy \\
11 & CB & SuperGLUE & NLI & various & accuracy \\
12 & COPA & SuperGLUE & QA & blogs, encyclopedia & accuracy \\
13 & BoolQA & SuperGLUE & boolean QA & Wikipedia & accuracy \\
14 & MultiRC & SuperGLUE & QA & various & accuracy \\
15 & IMDB & SuperGLUE & sentiment analysis & movie reviews & accuracy \\
\hline
\end{tabular}
\end{table*}

In our experiments, we utilize 15 datasets from standard classification and NLU benchmarks (as detailed in Table \ref{apptable:2}), as well as the generation-heavy TRACE benchmark~\cite{wang2023trace}. Specifically, the classification datasets are sourced from well-established suites such as the CL benchmark~\cite{zhang2015character}, GLUE~\cite{wang2018glue}, and SuperGLUE~\cite{wang2019superglue}, as well as the IMDB movie review dataset~\cite{maas2011learning}. These cover natural language inference (NLI), sentiment classification (SC), and topic classification (TC). 

Crucially, to evaluate generalization on more complex, generation-heavy tasks, we also incorporate the TRACE benchmark~\cite{wang2023trace}. Unlike standard classification tasks, TRACE features multi-turn, complex generative tasks that assess the structural reasoning and domain adaptation of LLMs. In particular, we evaluate on code generation (completion) using the Py150 dataset, where even a single syntax error can lead to a compilation failure, making it highly fragile under sequential fine-tuning.

\subsection{Task Sequence and Instructions}
\label{sec:appendix:sequences}

To evaluate the robustness of our method, we employ a diverse set of task instructions and evaluate across multiple task orders. Table~\ref{apptable:3} summarizes the specific natural language prompts used for each task category, following the instruction-tuning format described in \citet{qin2021lfpt5}. The task sequences for the T5 and Llama 2-7B experiments, together with the TRACE benchmark sequence, are provided in Table~\ref{apptable:4}.

\begin{table}[!htbp]
\caption{Task descriptions and corresponding prompts.}
\label{apptable:3}
\centering
\footnotesize
\setlength{\tabcolsep}{4pt}
\renewcommand{\arraystretch}{0.95}
\begin{tabularx}{\columnwidth}{@{}lX@{}}
\toprule
\textbf{Task} & \textbf{Prompts} \\
\midrule
NLI & What is the logical relationship between the ``sentence 1'' and the ``sentence 2''? Choose one from the option. \\
QQP & Whether the ``first sentence'' and the ``second sentence'' have the same meaning? Choose one from the option. \\
SC & What is the sentiment of the following paragraph? Choose one from the option. \\
TC & What is the topic of the following paragraph? Choose one from the option. \\
BoolQA & According to the following passage, is the question true or false? Choose one from the option. \\
MultiRC & According to the following passage and question, is the candidate answer true or false? Choose one from the option. \\
WiC & Given a word and two sentences, whether the word is used with the same sense in both sentences? Choose one from the option. \\
\bottomrule
\end{tabularx}
\end{table}

\begin{table}[!htbp]
\caption{Task sequences for different model and benchmark settings.}
\label{apptable:4}
\centering
\scriptsize
\setlength{\tabcolsep}{3pt}
\renewcommand{\arraystretch}{0.95}
\begin{tabularx}{\columnwidth}{@{}cl>{\raggedright\arraybackslash}X@{}}
\toprule
\textbf{Order} & \textbf{Setting} & \textbf{Task Sequence} \\
\midrule
1 & T5, Llama 2-7B & dbpedia $\rightarrow$ amazon $\rightarrow$ yahoo $\rightarrow$ ag \\
2 & T5, Llama 2-7B & dbpedia $\rightarrow$ amazon $\rightarrow$ ag $\rightarrow$ yahoo \\
3 & T5, Llama 2-7B & yahoo $\rightarrow$ amazon $\rightarrow$ ag $\rightarrow$ dbpedia \\
\midrule
4 & T5 & mnli $\rightarrow$ cb $\rightarrow$ wic $\rightarrow$ copa $\rightarrow$ qqp $\rightarrow$ boolqa $\rightarrow$ rte $\rightarrow$ imdb $\rightarrow$ yelp $\rightarrow$ amazon $\rightarrow$ sst-2 $\rightarrow$ dbpedia $\rightarrow$ ag $\rightarrow$ multirc $\rightarrow$ yahoo \\
5 & T5 & multirc $\rightarrow$ boolqa $\rightarrow$ wic $\rightarrow$ mnli $\rightarrow$ cb $\rightarrow$ copa $\rightarrow$ qqp $\rightarrow$ rte $\rightarrow$ imdb $\rightarrow$ sst-2 $\rightarrow$ dbpedia $\rightarrow$ ag $\rightarrow$ yelp $\rightarrow$ amazon $\rightarrow$ yahoo \\
6 & T5 & yelp $\rightarrow$ amazon $\rightarrow$ mnli $\rightarrow$ cb $\rightarrow$ copa $\rightarrow$ qqp $\rightarrow$ rte $\rightarrow$ imdb $\rightarrow$ sst-2 $\rightarrow$ dbpedia $\rightarrow$ ag $\rightarrow$ yahoo $\rightarrow$ multirc $\rightarrow$ boolqa $\rightarrow$ wic \\
\midrule
1 & TRACE & C-STANCE $\rightarrow$ FOMC $\rightarrow$ MeetingBank $\rightarrow$ Py150 $\rightarrow$ ScienceQA $\rightarrow$ NumGLUE-cm $\rightarrow$ NumGLUE-ds $\rightarrow$ 20Minuten \\
\bottomrule
\end{tabularx}
\end{table}

\end{document}